\documentclass[11pt, a4paper,USenglish]{article}

\usepackage{tikz}
\usetikzlibrary{arrows.meta}
\usepackage{graphicx}
\usepackage{subcaption}
\usepackage{amsmath,amssymb,amsthm}
\usepackage{mathtools}
\usepackage[ruled,vlined]{algorithm2e}
\usepackage{tcolorbox}
\usepackage{pgfplots}
\usepackage{hyperref}
\usepackage{subfiles} 
\usepackage{pdflscape}
\usepackage{cleveref}
\usepackage{float}

\usepackage[margin=1in]{geometry} % Standard 1-inch margins
\usepackage{setspace}
\theoremstyle{plain}
\newtheorem{theorem}{Theorem}[section] % Numbered per section (e.g., Theorem 1.1)
\newtheorem{lemma}[theorem]{Lemma}     % Shares numbering with Theorem
\newtheorem{proposition}[theorem]{Proposition}
\newtheorem{corollary}[theorem]{Corollary}
\newtheorem{conjecture}[theorem]{Conjecture}

\theoremstyle{definition}
\newtheorem{definition}[theorem]{Definition}
\newtheorem{example}[theorem]{Example}

\theoremstyle{remark}
\newtheorem{remark}[theorem]{Remark}
\newtheorem*{remark*}{Remark}           % Unnumbered version

\tikzstyle{wViolet}=[circle, draw, fill=violet, inner sep=0pt, minimum width=5pt]
\tikzstyle{wWhite}=[circle, draw, fill=white, inner sep=0pt, minimum width=5pt]

\newcommand{\gcr}{\mathrm{gcr}}
\newcommand{\mlt}{\mathrm{mlt}}
\newcommand{\wmlt}{\mathrm{wmlt}}
\newcommand{\VR}{\mathcal{VR}}

\title{A computational approach to maximum likelihood thresholds for colored Gaussian graphical models}

\author{
  Roser Homs\thanks{Universitat Politècnica de Catalunya - BarcelonaTech (UPC), \href{mailto:roser.homs@upc.edu}{roser.homs@upc.edu}}
  \and
  Olga Kuznetsova\thanks{Aalto University, \href{mailto:kuznetsova.olga@gmail.com}{kuznetsova.olga@gmail.com}}
  \and
  Bernadette J. Stolz\thanks{Department of Machine Learning and Systems Biology, Max Planck Institute of Biochemistry; Munich Center for Machine Learning; Laboratory for Topology and Neuroscience, School of Life Sciences, École Polytechnique Fédérale de Lausanne, 1015 Lausanne, Switzerland, \href{mailto:stolz@biochem.mpg.de}{stolz@biochem.mpg.de}}
}

\makeatletter
\let\old@includegraphics\includegraphics
\renewcommand{\includegraphics}[2][]{%
  \def\temp{#2}%
  \ifx\temp\ccbyname\else
    \old@includegraphics[#1]{#2}%
  \fi
}
\def\ccbyname{cc-by}
\makeatother 

\begin{document}

\maketitle

\begin{abstract}
Gaussian graphical models (GGMs) are essential tools for interpretable structure learning. However, in high-dimensional, small-sample regimes, the available data is often insufficient for the maximum likelihood estimator to exist. Colored Gaussian graphical models (CGGMs) mitigate this limitation by imposing symmetry constraints through graph coloring, which reduces the required sample size. This minimal number of observations needed to guarantee that the estimator exists almost surely is defined as the maximum likelihood threshold (MLT). Here, we address the computation of the MLT for CGGMs by focusing on its geometric formulation: finding the minimum rank of a sample covariance matrix such that its projection lies almost surely within the interior of the cone of sufficient statistics. We establish a unified theoretical framework, extending results from uncolored to colored models and introducing new symbolic algorithms. Furthermore, we present a computational study integrating sampling with topological data analysis (TDA) to investigate the local geometry of the cone of sufficient statistics. Our results demonstrate the potential of TDA to overcome the computational bottlenecks of traditional symbolic algebraic methods–particularly Gröbner basis computations–in analyzing the likelihood geometry of CGGMs.
\end{abstract}

\section{Introduction}

Gaussian graphical models (GGMs) describe multivariate normal distributions whose conditional independence structure is encoded by a graph: each node represents a random variable, and missing edges correspond to conditional independences. A fundamental methodological question underlying the application of GGMs concerns sample size: what is the minimal number of observations guaranteeing that the maximum likelihood estimator (MLE) exists almost surely? This number is called the maximum likelihood threshold (MLT) and is relevant in the fields characterized by the high-dimensional, small-sample regime, such as the study of biological networks, e.g. gene regulatory networks \cite{zhao2019cancer,wille2004sparse}, biochemical pathways \cite{krumsiek2011gaussian}, protein interactions \cite{wang2016fastggm}, multi-omics \cite{shutta2023dragon}.

In \emph{colored Gaussian graphical models} (CGGMs), introduced in \cite{Lauritzen2008}, vertex or edge colors impose equality among the corresponding entries in the inverse of the covariance matrix. These symmetry constraints often lead to a reduction in the MLT.

For uncolored chordal graphs, the MLT coincides with the size of the maximal clique \cite{Buhl}. In the general uncolored setting, ML thresholds have been studied using algebraic and convex geometry \cite{blekherman2019maximum,Uhler}, algebraic matroids \cite{gross2018maximum}, rigidity theory \cite{Bernstein2024Rigidity,Bernstein2022MLT}, and score matching estimators \cite{forbes2015linear}. Nevertheless, for non-chordal or colored graphs, the behavior of ML thresholds remains poorly understood, and efficient computational approaches are lacking. 

Here, we address the geometric formulation of the MLT problem for CGGMs: what is the minimum rank of a sample covariance matrix such that its sufficient statistics lie almost surely in the interior of the cone of sufficient statistics? The contributions of this paper are:

\begin{itemize}
\item A unified theoretic framework for ML thresholds of CGGMs and related quantities (weak MLT and generic completion rank), extending existing results from uncolored to colored models, introducing new symbolic algorithms that help resolve open cases, and standardizing notation.
\item A computational study that introduces  topological data analysis (TDA) computed locally as a method to study (weak) MLT. Our work demonstrates that TDA-based approaches have the potential to overcome the computational limitations of symbolic algebraic methods, particularly those relying on Gröbner basis computations, when investigating the likelihood geometry of colored Gaussian graphical models. At the same time, it highlights the challenge of obtaining representative sufficient statistics for computational studies, as those arising from small-sample observations tend to concentrate near the boundary of the cone of sufficient statistics rather than exploring its interior. We address this challenge by developing an algorithm that emulates uniform sampling by tracing the map between positive semidefinite matrices and their projection into the cone of sufficient statistics.
\end{itemize}

We investigate the performance of our algorithms in the case of the colored $4$-cycle, which has the benefit of having been studied extensively in \cite{Uhler}, while also presenting challenges which have not been overcome by existing methods. Together, these results strengthen the largely underexplored connection between algebraic statistics and topological data analysis.

\section{Geometry of MLE existence for Colored Gaussian Graphical Models}
\label{sec:likelihood geometry}

In this section, we introduce the theoretical background underlying CGGMs. Our goal is to provide a formal description of the geometry that governs the existence of MLE. For background on statistical graphical models and their algebraic treatment, see \cite[Chapter 5]{maathuis2018handbook}. 

\subsection{Colored Gaussian graphical models} \label{section: colored Gaussian graphical models}
Let \(\mathrm{Sym}(m)\) denote the vector space of real symmetric \(m\times m\) matrices and let \(\mathrm{PD}(m)\) (resp. \(\mathrm{PSD}(m)\)) be the full-dimensional open (resp. closed) convex cone of \(m\times m\) positive definite (resp. semidefinite) matrices.
We denote by \(\mathrm{Sym}(m,n)\) and \(\mathrm{PSD}(m,n)\), with \(n<m\), the set of \(m\times m\) symmetric and positive semidefinite matrices, resp., of rank \(\leq n\). 
Consider a random vector \(X \in \mathbb{R}^m\) following a centered multivariate Gaussian distribution with covariance matrix
\(\Sigma \in \mathrm{PD}(m)\). Let $K$ denote the concentration matrix \(K=\Sigma^{-1}\).

\begin{definition}
\label{def:coloredGraph}
A \emph{colored graph} \(G=(V,E,\lambda)\) is an undirected graph on vertex set \(V = [m]\) and edge set \(E\) together with a coloring function \(\lambda: V\sqcup E \to [d]\) such that \(\lambda(V)\cap\lambda(E)=\emptyset\).
\end{definition}

Colored graphs encode equalities among the entries of the concentration matrix \(K = \Sigma^{-1}\). 
 Note that the coloring function in \Cref{def:coloredGraph} is not allowed to assign the same color to both vertices and edges. See \cite[Section 3.1]{Lauritzen2008} for more details on the statistical interpretation of the vertex and edge symmetry constraints imposed by the coloring. 
 
\begin{definition}
Consider the \(d\)-dimensional linear subspace 
\begin{multline}
 \mathcal{L}_G = \{K=(k_{ij}) \in \mathrm{Sym}(m) \;|\;
k_{ij}=0 \text{ for } (i,j)\notin E,\\
 k_{ii} = k_{jj} \text{ if }  \lambda(i)= \lambda(j) \text{ for }i,j\in V,\\
\text{ and } k_{ij} = k_{lm} \text{ if }  \lambda(i,j)= \lambda(l,m)  \text{ for }(i,j),(l,m)\in E\}.   
\end{multline}
The \emph{cone of concentration matrices} associated to the colored graph \(G\) is \(\mathcal{K}_G=\mathcal{L}_G\cap\mathrm{PD}(m)\).
\end{definition}

\begin{remark}\label{remark: relatively open cone}
The cone of concentration matrices \(\mathcal{K}_G=\mathcal{L}_G\cap\mathrm{PD}(m)\) is a relatively open convex cone in \(\mathcal{L}_G\). By considering a basis \(K_1,\dots,K_d\) of  \(\mathcal{L}_G\), one can identify \(\mathcal{K}_G\) with the open convex cone
\(\lbrace (\lambda_1,\dots,\lambda_d)\in\mathbb{R}^d: \lambda_1K_1+\dots+\lambda_dK_d\in\mathrm{PD}(m)\rbrace\)
of \(\mathbb{R}^d\).
For more details on the underlying convex (algebraic) geometry, see \cite[Section 2]{UhlerSturmfels}.
\end{remark}

\begin{definition} A \emph{colored Gaussian graphical model} is the set of covariance matrices 
\(\mathcal{M}_G = \{\Sigma \in \mathrm{PD}(m) \;|\; \Sigma^{-1} \in \mathcal{K}_G\}.\)
\end{definition}

\begin{remark}
Note that the model $\mathcal{M}_G$ is the intersection of the algebraic variety \(\mathcal{L}_G^{-1}\) with the cone of positive definite matrices $\mathrm{PD}(m)$, where \(\mathcal{L}_G^{-1}\) denotes the Zariski closure of $\lbrace \Sigma\in\mathrm{Sym(m):\det\Sigma\neq 0 \mbox{ and }\Sigma^{-1}\in\mathcal{L}_G}\rbrace$.
\end{remark}

\begin{figure}[h]
\centering

%---------------- LEFT: ----------------%
\begin{minipage}{0.45\textwidth}
\centering
\begin{tikzpicture}
\draw[thick, cyan] (0,0) -- (2,0);
 
	\node[left] at (0,0) {1};
    \node[right] at (2,0) {2};
	
\node[wViolet] at (0,0) {};
\node[wViolet] at (2,0) {};
\end{tikzpicture}
\end{minipage}
\hfill
%---------------- RIGHT: ----------------%
\begin{minipage}{0.45\textwidth}
\centering
\begin{tikzpicture}[scale=1.2]

    % vertices
    \node (v1) at (0,0)    {};
    \node (v2) at (1,0)    {};
    \node (v3) at (1,1)    {};
    \node (v4) at (0,1)    {};

    % labels
    \node[left]  at (v1) {4};
    \node[right] at (v2) {3};
    \node[right] at (v3) {2};
    \node[left]  at (v4) {1};

    % colored edges
    \draw[thick, cyan]   (v1) -- (v2);
    \draw[thick, red]    (v2) -- (v3);
    \draw[thick, green]  (v3) -- (v4);
    \draw[thick, gray] (v4) -- (v1);

    % colored vertices
    \node[wViolet] at (v1) {};
    \node[wViolet] at (v2) {};
    \node[wViolet] at (v3) {};
    \node[wViolet] at (v4) {};
\end{tikzpicture}
\end{minipage}
\caption{2-cycle and 4-cycle with equally colored vertices and all edges colored differently. See \Cref{table:3cycles} and \Cref{table: all graphs} for examples of other colorings on 3 and 4-cycles.}
\label{fig:cycles}
\end{figure}
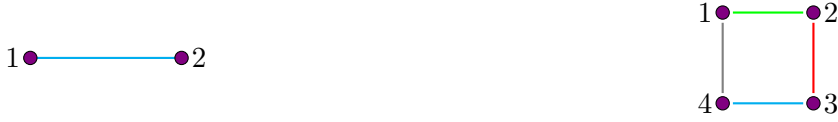

\begin{example}\label{ex:cycles} Consider the colored 2-cycle \(C_2\) and 4-cycle \(C_4\) in \Cref{fig:cycles}. Then
\[\mathcal{L}_{C_2}=\left\lbrace\begin{pmatrix}
     \lambda_1 & \lambda_2\\
     \lambda_2 & \lambda_1
    \end{pmatrix}: \lambda_i\in\mathbb{R},\,i\in[2]\right\rbrace\text{ and }
\mathcal{L}_{C_4}=\left\lbrace\begin{pmatrix}
     \lambda_1 & \lambda_2 & 0 & \lambda_5\\
     \lambda_2 & \lambda_1 & \lambda_3 & 0\\
     0 & \lambda_3 & \lambda_1 & \lambda_4\\
     \lambda_5 & 0 & \lambda_4 & \lambda_1
    \end{pmatrix}: \lambda_i\in\mathbb{R},\,i\in[5]\right\rbrace.\]

\end{example}

\subsection{The cone of sufficient statistics}
Given \(n\) independent and identically distributed observations \(X_1,\dots,X_n\) drawn from a centered multivariate Gaussian distribution, we consider their sample covariance matrix
\(
S = \frac{1}{n}\sum_{i=1}^n X_i X_i^{\mathsf{T}}.\)
Up to additive constants, its log-likelihood function is
\begin{equation}\label{eq:loglike}
\ell_S(K) = \log\det K - \langle S, K \rangle,
\end{equation}
\noindent
 where \(\langle S,K \rangle = \mathrm{tr}(SK)\) is the trace inner product on \(\mathrm{Sym}(m)\). We describe the likelihood as a function of the concentration matrix because \(\ell_S(K)\) is concave on \(\mathrm{PD}(m)\) as opposed to \(\ell_S(\Sigma)\), see \cite[Proposition 3.1]{uhler2018gaussian}. Moreover, since \(\mathcal{K}_G\) is a convex cone, then for $S\in\mathrm{PD}(m)$,
 \begin{equation}\label{eq:mle}
  \begin{array}{ll}
\text{maximize} & \log\det K - \langle S, K \rangle,\\
\text{subject to} & K \in \mathcal{K}_G.
\end{array}  
\end{equation}
\noindent is a convex optimization problem with a unique maximum \(\hat{K}\), the MLE for the concentration matrix. The MLE for the covariance matrix is \(\hat{\Sigma}=\hat{K}^{-1}\). See \cite[Section 3]{uhler2018gaussian} for an introduction to Gaussian likelihood and convex optimization.

Given a basis \(K_1,\dots,K_d\) of \(\mathcal{L}_G\), one can decompose \(\langle S,K\rangle=\sum_{i=1}^d\lambda_i\langle S,K_i\rangle\), where \(K=\sum_{i=1}^d\lambda_iK_i\) and $\lambda_i\in\mathbb{R}$. Therefore, all information from the data required to perform the convex optimization problem \eqref{eq:mle} is encoded in the projection map
\begin{equation}\label{eq:proj}
\pi_G : \mathrm{Sym}(m) \longrightarrow \mathbb{R}^d, \qquad
S \longmapsto (\langle S,K_1\rangle, \dots, \langle S,K_d\rangle).
\end{equation}

\begin{example}\label{ex:c4proj}
Consider the 4-cycle \(C_4\) with equally colored vertices in \Cref{ex:cycles}. Choose the basis \(K_1,\dots,K_5\) of \(\mathcal{L}_{C_4}\) where $K_i$ is given by $\lambda_j=\delta_{ij}$.
Then
\[
\pi_G : \mathrm{Sym}(4) \longrightarrow \mathbb{R}^5, \qquad
S=(s_{ij}) \longmapsto (\mathrm{tr}(S),2s_{12},2s_{14},2s_{23},2s_{34}).
\]
\end{example}

\begin{definition} The \emph{fiber of \(S\) with respect to \(G\)} is 
\[\operatorname{fiber}_G(S)
   = \{S' \in \mathrm{Sym}(m) \;|\;
      \pi_G(S')=\pi_G(S)\}.\]
\end{definition}
\noindent
For any $S'\in\operatorname{fiber}_G(S)$,  
$\langle S',K\rangle=\langle S,K\rangle$ for any $K\in\mathcal{L}_G$ or, equivalently, $\ell_{S'}=\ell_S$ on $\mathcal{L}_G$.

\begin{definition}
The \emph{cone of sufficient statistics} is the open convex cone \(
\mathcal{C}_G := \pi_G(\mathrm{PD}(m))\) in \(\mathbb{R}^d\).
\end{definition}

As any sample covariance matrix is a positive semi-definite matrix by construction, we seek to study the projection of the PSD cone. As established in \cite{UhlerSturmfels}, the euclidean closure of the cone of sufficient statistics satisfies \(\mathrm{cl}(\mathcal{C}_G)=\pi_G(\mathrm{PSD}(m))\). Projections of PD matrices are interior points of \(\mathcal{C}_G\). 
To precisely characterize when the projections of PSD matrices yield interior versus boundary points in this cone, we include the following result derived from fundamental principles of convex geometry.

\begin{lemma}\label{lem:boundary}
Let $S\in\mathrm{PSD}(m)$. Then
$\pi_G(S)$ is on the boundary of $\mathcal{C}_G$ iff $\langle S,K\rangle=0$ for some non-zero $K\in\mathcal{L}_G\cap\mathrm{PSD}(m)$. Otherwise $\pi_G(S)$ is in the interior.
\end{lemma}

\begin{proof}
The cone of sufficient statistics \(\mathcal{C}_G\) can be viewed inside \(\mathrm{Sym}(m)/\mathcal{L}_G^\perp\simeq\mathbb{R}^d\) as the open convex dual of the cone of concentration matrices \(\mathcal{K}_G=\mathcal{L}_G\cap\mathrm{PD}(m)\), namely \[\mathcal{C}_G=\lbrace S\in\mathrm{Sym}(m)/\mathcal{L}_G^\perp: \langle S,K\rangle>0 \mbox{ for all }K\in\mathcal{L}_G\cap\mathrm{PD}(m)\rbrace=\mathcal{K}_G^\ast\] and 
\(\mathrm{cl}(\mathcal{C}_G)=\lbrace S\in\mathrm{Sym}(m)/\mathcal{L}_G^\perp: \langle S,K\rangle\geq 0 \mbox{ for all }K\in\mathcal{L}_G\cap\mathrm{PSD}(m)\rbrace=\mathrm{cl}(\mathcal{K}_G)^\ast.\)
Therefore, the boundary of the convex cone of sufficient statistics can be described as  
\(\partial\mathcal{C}_G=\lbrace S\in\mathrm{Sym}(m)/\mathcal{L}_G^\perp: \langle S,K\rangle=0 \mbox{ for some non-zero }K\in\mathcal{L}_G\cap\mathrm{PSD}(m)
\rbrace\).
\end{proof}

\begin{remark}\label{rem:convex} Note that in some references \cite[Chapter 2.6]{boyd2004convex}, the convex dual $\mathcal{K}_G^\ast$ is defined as $\mathrm{cl}(\mathcal{C}_G)$ and $\mathcal{C}_G$ is referred to as $\mathrm{int}(\mathcal{K}_G^\ast)$. For more details, see \cite[Sections 1,2]{UhlerSturmfels} and \cite[Chapter 8.3]{sullivant2018algebraic}.
\end{remark}

\subsection{Existence of the MLE}

The following classical result describes existence of the MLE geometrically, see \cite[Theorem 5.2]{uhler2018gaussian}:

\begin{theorem}\label{theorem: characterisation of MLE}
The MLE \(\hat{\Sigma}\) exists if and only if \(\operatorname{fiber}_G(S)\cap\mathrm{PD}(m)\neq \emptyset\).  
If it exists, \(\hat{\Sigma}\) is the unique matrix in the fiber satisfying
\(\hat{\Sigma}^{-1} \in \mathcal{K}_G\).
\end{theorem}

An immediate corollary is that the MLE exists for a sample covariance matrix $S$ if and only if \(\pi_G(S)\) lies in the interior of the cone of sufficient statistics \(\mathcal{C}_G\). In particular, the MLE always exists for any \(S\in\mathrm{PD}(m)\). 
By construction, a sample covariance matrix \(S\) is always positive semidefinite, hence \(\pi_G(S) \in \mathrm{cl}(\mathcal{C}_G)\).  

\begin{tcolorbox}[title=MLE existence test via the cone of sufficient statistics]\label{box:MLEtest}
Given a colored GGM \(\mathcal{M}_G\) and a sample covariance matrix \(S\in\mathrm{PSD}(m)\):
\begin{itemize}
    \item \(\pi_G(S)\in\mathcal{C}_G\Leftrightarrow\) the MLE exists for $S$
    \item \(\pi_G(S)\in\partial\mathcal{C}_G\Leftrightarrow\) the MLE does not exist for $S$
\end{itemize}
\end{tcolorbox}

\begin{definition}
The \emph{maximum likelihood threshold} of \(G\), \(\operatorname{mlt}(G)\), is the smallest \(n\) such that the MLE exists with probability one for sample covariance matrices of rank \(n\). The \emph{weak maximum likelihood threshold}, \(\operatorname{wmlt}(G)\), is the smallest  \(n\) such that the MLE exists with positive probability for samples of rank \(n\).
\end{definition}

\begin{remark}
Generically, the sample covariance matrix of $n$ observations has rank \(n\). Throughout the paper, we refer interchangeably to “rank-$n$
sample covariance matrices” and to “samples of $n$ generic observations” to connect the geometric constraint $S \in PSD(m,n)$ with its statistical sampling interpretation. Moreover, we call a statement \emph{generic} when it holds on a Zariski-open dense set (stronger) to distinguish it from probabilistic \emph{almost sure} statements that hold everywhere except for a positive measure set (weaker). Check \cite[Appendix A]{Bernstein2024Rigidity} for a discussion on the relation among the two notions.
\end{remark}

\begin{tcolorbox}[title=Geometric characterization of ML thresholds]\label{box:MLT}
\begin{itemize}
    \item \(\mlt(G)=\) minimum \(n\) for which \(\pi_G(S)\in\mathcal{C}_G\) for almost every \(S\in\mathrm{PSD}(m,n)\),
    \item \(\wmlt(G)=\) minimum \(n\) for which  \(\pi_G(S)\in\mathcal{C}_G\) for some \(S\in\mathrm{PSD}(m,n)\).
\end{itemize}
\end{tcolorbox}

\begin{remark} Note that existence of the MLE for some $S$ immediately implies the existence of an open set of positive measure that contains a PD matrix $\Sigma\in\operatorname{fiber}_G(S)$ for which MLE exists, see \cite[Definition 5.1]{gross2018maximum}. 
\end{remark}

\begin{example}\label{ex:suffStats} Consider the colored 2-clique \(C_2\) in \Cref{ex:cycles}. In this case, \(\mathcal{L}_{C_2}^{-1}=\mathcal{L}_{C_2}\) and hence the model $\mathcal{M}_{C_2}$ is a convex cone as shown in \Cref{fig:likelihood_geometry_2cycle}. The projection map $\pi_G:\mathrm{Sym}(2)\rightarrow\mathbb{R}^2$ is defined by \(\pi_{C_2}(S)=(s_{11}+s_{22},2s_{12})\), thus 
the cone of sufficient statistics is \(\mathcal{C}_{C_2}=\lbrace(x,y)\in\mathbb{R}^2: 0< \vert y\vert<x\rbrace\).  Since \(\pi_{C_2}(\mathrm{PSD}(2,1))=\pi_{C_2}(\mathrm{PSD}(2))=\mathrm{cl}(\mathcal{C}_{C_2})\), then the projection of a rank 1 matrix lies in the interior of $\mathcal{C}_{C_2}$ almost surely and \(\mlt({C_2})=1\). 
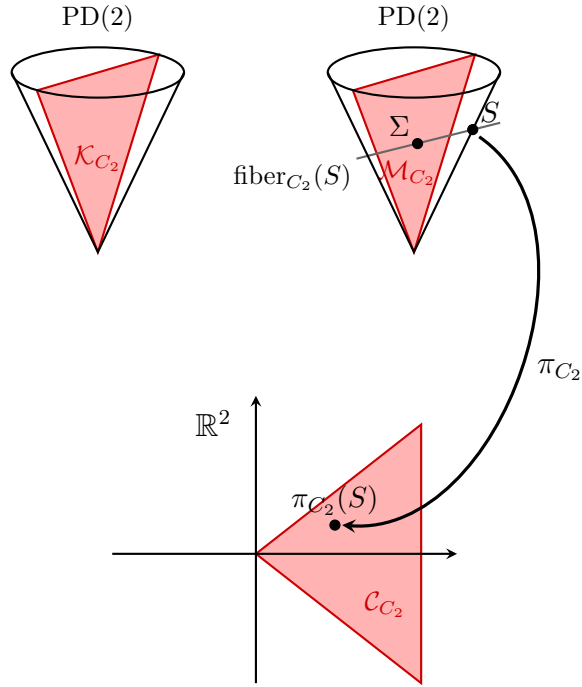
\begin{figure}[htbp]
\centering
\begin{tikzpicture}[>=stealth, scale=0.95]

    % ==========================================
    % 1. TOP-LEFT CONE: Cone with red slice
    % ==========================================
    % Shifted closer to the center (x: -4 -> -2.2, y: 5 -> 4.2)
    \begin{scope}[shift={(-2.2,4.2)}]
        % --- Step 1: Draw filled slice FIRST ---
        \fill[red!30, draw=red!80!black, thick] (0,0) -- (-0.85,2.25) -- (0.85,2.75) -- cycle;
        
        % Label for slice: \mathcal{K}_G
        \node[red!80!black, font=\small] at (0, 1.3) {$\mathcal{K}_{C_2}$};

        % --- Step 2: Draw cone outline SECOND ---
        \draw[thick] (0,2.5) circle [x radius=1.2, y radius=0.35];
        \draw[thick] (-1.2,2.5) -- (0,0) -- (1.2,2.5);

        % Label for top cone: \text{PD}_2 (or \text{PD}(2))
        \node[above=2pt, font=\small] at (0, 2.85) {$\text{PD}(2)$};
    \end{scope}

    % ==========================================
    % 2. TOP-RIGHT CONE: Fiber line, S, \Sigma
    % ==========================================
    % Shifted closer to the center (x: 2 -> 2.2, y: 5 -> 4.2)
    \begin{scope}[shift={(2.2,4.2)}]
        % --- Step 1: Draw filled slice FIRST ---
        \fill[red!30, draw=red!80!black, thick] (0,0) -- (-0.85,2.25) -- (0.85,2.75) -- cycle;
        
        % Label for slice: \mathcal{M}_G
        \node[red!80!black, font=\small] at (-0.1, 1.1) {$\mathcal{M}_{C_2}$};

        % --- Step 2: Fiber Line & Points ---
        \draw[thick, gray!80!black] (-0.8,1.3) -- (1.2,1.8);
        
        % \Sigma inside the slice
        \filldraw[black] (0.05, 1.51) circle (2pt) node[above left=-1pt] {$\Sigma$};
        
        % S on the outer boundary of the cone
        \filldraw[black] (0.818, 1.705) circle (2pt) node[above right=-1pt] {$S$};

        % Label for Fiber (below left of line endpoint)
        \node[below left=-2pt, font=\small] at (-0.8, 1.3) {$\text{fiber}_{C_2}(S)$};

        % --- Step 3: Draw cone outline SECOND ---
        \draw[thick] (0,2.5) circle [x radius=1.2, y radius=0.35];
        \draw[thick] (-1.2,2.5) -- (0,0) -- (1.2,2.5);

        % Label for top cone: \text{PD}_2
        \node[above=2pt, font=\small] at (0, 2.85) {$\text{PD}(2)$};

        % Coordinate anchor for the projection arrow (anchored at S)
        \coordinate (FiberStart) at (0.908, 1.605);
    \end{scope}

    % ==========================================
    % 3. BOTTOM R^2 PLANE: Projected Cone C_G
    % ==========================================
    \begin{scope}[shift={(0,0)}]
        \fill[red!30, draw=red!80!black, thick] (0,0) -- (2.3, 1.8) -- (2.3, -1.8) -- cycle;

        % Label for bottom cone: \mathcal{C}_G
        \node[red!80!black, font=\small] at (1.8, -0.7) {$\mathcal{C}_{C_2}$};

        % Axes R^2
        \draw[->, thick] (-2,0) -- (2.8,0);
        \draw[->, thick] (0,-1.8) -- (0,2.2) node[above] {};
        \node[above left, font=\large] at (-0.2,1.5) {$\mathbb{R}^2$};

        % Point \pi_G(S)
        \filldraw[black] (1.1, 0.4) circle (2pt) node[above] {$\pi_{C_2}(S)$};

        % Coordinate anchor for the receiving arrow
        \coordinate (ProjEnd) at (1.2, 0.4);
    \end{scope}

    % ==========================================
    % PROJECTION MAP ARROW (\pi_G)
    % ==========================================
    \draw[->, very thick, bend left=75] (FiberStart) to node[right=2pt, font=\large] {$\pi_{C_2}$} (ProjEnd);

\end{tikzpicture}
\caption{Diagram representing the likelihood geometry of the colored 2-cycle $C_2$.}
\label{fig:likelihood_geometry_2cycle} 
\end{figure}
\end{example}

\section{Computational algebraic geometry approaches to ML thresholds}\label{sec:alggeo}
Given a colored Gaussian graphical model $\mathcal{M}_G$, a natural way to study existence of the MLE for $n$ observations in $\mathbb{R}^m$ is to study $\pi_G(\mathrm{PSD}(m,n))$, as in the colored 2-cycle from \Cref{ex:suffStats}. However, providing an explicit description of this set becomes infeasible already for very small graphs, such as the colored 4-cycle from \Cref{ex:cycles}. To take advantage of the machinery from algebraic geometry, instead of studying the projection of the convex cone $\mathrm{PSD}(m,n)$, one can study the projection of the algebraic variety \(\mathrm{Sym}(m,n)\).

\subsection{Algebraic relaxation: from convex cones to algebraic varieties}\label{subsec:gcr}

\begin{definition}\label{def:gcr}
   The \emph{generic completion rank of $G$}, $\gcr(G)$, is the minimum $n$ for which \(\dim \pi_G(\mathrm{Sym}(m,n))=d\).
\end{definition}

\begin{remark}\label{rem:low-ramk}
Although the algebraic relaxation to tackle the MLE existence problem was first introduced by Uhler in \cite{Uhler}, the previous definition appeared for the first time in \cite{gross2018maximum}, under the name of rank of $G$. The terminology $\gcr(G)$ was later adopted in \cite{blekherman2019maximum}, given the connection to low rank matrix completion problems in the uncolored setting. For more details on the equivalence of MLE existence and completion problems for colored graphs, see \Cref{subsec:CompletionProblems}.
\end{remark}

The following necessary condition relating the number of color classes $d$ to the generic completion rank follows immediately from the definition:

\begin{theorem}\label{thm:bounds_gcr}
    If \(\gcr(G)=n\), then \(d \leq nm-\binom{n}{2}\).
\end{theorem}

\begin{proof} By definition of \(gcr(G)\), %a dimension count yields
\(d=\dim \pi_G(\mathrm{Sym}(m,n))\leq \dim \mathrm{Sym}(m,n)=nm-n(n-1)/2\).
In the uncolored case $d=m+\vert E\vert$, recovering \cite[Theorem 1.4]{blekherman2019maximum}.
\end{proof}

As is commonplace in algebraic geometry, we study the complexification of the projection map, namely $\pi_G^\mathbb{C}:\mathrm{Sym}_\mathbb{C}(m)\rightarrow \mathbb{C}^d$. The definition of $\gcr(G)$ is independent of the underlying field: $\dim_\mathbb{R} \pi_G\left(\mathrm{Sym}_\mathbb{R} (m,n)\right)=\dim_\mathbb{C} \pi_G\left(\mathrm{Sym}_\mathbb{C} (m,n)\right)$. For simplicity, we omit $\mathbb{R}$ or $\mathbb{C}$ from our notation.
Working over $\mathbb{C}$ provides two features that do not hold over $\mathbb{R}$: 
\begin{itemize}
    \item For any \(S\in\mathrm{Sym}(m,n)\), there exists $X\in\mathbb{C}^{m\times n}$ such that $S=XX^t$. 
    \item The condition $\dim \pi_G\left(\mathrm{Sym}(m,n)\right)=d$ is equivalent to $\pi_G\left(\mathrm{Sym}(m,n)\right)$ being dense in the target space with respect to Euclidean topology.
\end{itemize}
Therefore, the restriction of $\pi_G$ to \(\mathrm{Sym}(m,n)\) is dominant (the image of the projection has maximal dimension) if and only if its differential $d_S\pi_G$ is surjective when evaluated at a generic point $S=XX^t$. Given that surjectivity of $d_S\pi_G$ is a rank condition, it is enough that the corresponding Jacobian matrix is full rank at a single point to ensure that it holds generically.

\begin{tcolorbox}[title=Geometric characterization of all thresholds]\label{box:all}
\begin{itemize}
    \item \(\gcr(G)=\) minimum \(n\) for which  \(d_S\pi_G\) is surjective for some \(S\in\mathrm{PSD}(m,n)\).
    \item \(\mlt(G)=\) minimum \(n\) for which \(\pi_G(S)\in\mathcal{C}_G\) for almost every \(S\in\mathrm{PSD}(m,n)\).
    \item \(\wmlt(G)=\) minimum \(n\) for which  \(\pi_G(S)\in\mathcal{C}_G\) for some \(S\in\mathrm{PSD}(m,n)\).
\end{itemize}
\end{tcolorbox}

Let us now characterize the three previous thresholds using a framework that enables direct comparison among them. The following result adapts and generalizes \cite[Proposition 1.2 (2), Proposition 1.6 (2)]{blekherman2019maximum} of Blekherman-Sinn. For the sake of completeness, we include a variant of the original proof. 

\begin{theorem}\label{th:Sinn2} 
Let $G$ be a colored graph with vertex set $V=[m]$. Then
    \begin{enumerate}
        \item \(\gcr(G)\) is the smallest \(n\) for which there exists an $m\times n$ matrix $X$ with the property that $\mathrm{im}(X)\not\subset\ker(K)$ for any \(K\in\mathcal{L}_G\backslash\lbrace 0\rbrace\).
        \item \(\mlt(G)\) is the smallest \(n\) for which a generic $m\times n$ matrix $X$ satisfies $\mathrm{im}(X)\not\subset\ker(K)$ for any \(K\in\mathcal{L}_G\cap\mathrm{PSD}(m)\backslash\lbrace 0\rbrace\).
        \item \(\wmlt(G)\) is the smallest \(n\) for which there exists an $m\times n$ matrix $X$ with the property that $\mathrm{im}(X)\not\subset\ker(K)$ for any \(K\in\mathcal{L}_G\cap\mathrm{PSD}(m)\backslash\lbrace 0\rbrace\).
    \end{enumerate}
\end{theorem}

\begin{proof}
Consider the characterization of $\gcr(G)$ in (\ref{box:all}). Take $S=XX^t$ for some $X\in\mathbb{C}^{m\times n}$ and let us denote by $T_SV_n$ the tangent space to $V_n=\mathrm{Sym}(m,n)$ at $S$. The differential $d_S\pi_G:T_SV_n\rightarrow \mathbb{C}^d$ is surjective iff the dual map $d_S\pi_G^\ast: \left(\mathbb{C}^d\right)^\ast\rightarrow \left(T_SV_n\right)^{\ast}$ is injective.
The dual space $\left(\mathbb{C}^d\right)^\ast$ is identified with $\mathcal{L}_G$ via the trace pairing, hence $d_S\pi_G^\ast$ fails to be injective iff there is a non zero-matrix $K\in\mathcal{L}_G$ such that $\langle K, M\rangle=0$ for any $M\in T_SV_n$. 
One can check that $T_SV_n=\lbrace XA^t+AX^t: A\in\mathbb{C}^{m\times n}\rbrace$, and
$\langle K, XA^t+AX^t\rangle=0$ for any $A\in\mathbb{C}^{m\times n}$ iff $KX=0$.

Finally, the statements about $\mlt(G)$ and $\wmlt(G)$ follows from the description of boundary and interior points in \Cref{lem:boundary}. Indeed, any $S\in\mathrm{PSD}(m,n)$ can be written as $S=XX^t$ for some $X\in\mathbb{R}^{m\times n}$. Because of positive semi-definiteness of both $S$ and $K$, the condition $\langle S,K\rangle=0$ is equivalent to $SK=0$. One can check that $SK=0$ is equivalent to $\mathrm{im}(X)\subset\ker(K)$. Therefore, $\pi_G(S)$ is in the interior of $\mathcal{C}_G$ if and only if $\mathrm{im}(X)\not\subset\ker(K)$ holds for any non-zero $K\in\mathcal{L}_G\cap\mathrm{PSD}(m)$.
\end{proof}

\begin{remark}
    Note that the matrix $X$ in the statements of \Cref{th:Sinn2} can be regarded as a matrix whose column vectors are the observations $X_1,\dots,X_n$. The matrix $S=XX^t$ is the sample covariance matrix up to scaling.
\end{remark}

Instead of giving conditions on the kernels of the covariance matrices of our model as in \Cref{th:Sinn2}, one can alternatively characterize the thresholds by giving conditions on the concentration matrices themselves:

\begin{theorem}\label{th:alt} Let $G$ be a colored graph with vertex set $V=[m]$. Then
    \begin{enumerate}
        \item \(\gcr(G)\) is the smallest \(n\) such that there exist an $m\times(m-n)$ matrix $V$ that satisfies $VAV^t\notin\mathcal{L}_G$ for any non-zero $A\in\mathrm{Sym}(m-n)$.
        \item \(\mlt(G)\) is the smallest \(n\) such that a generic $m\times(m-n)$ matrix $V$ satisfies $VAV^t\notin\mathcal{L}_G$ for any non-zero $A\in\mathrm{PSD}(m-n)$.
        \item \(\wmlt(G)\) is the smallest \(n\) such that there exists an $m\times(m-n)$ matrix $V$ that satisfies $VAV^t\notin\mathcal{L}_G$ for any non-zero $A\in\mathrm{PSD}(m-n)$.
    \end{enumerate}
\end{theorem}

\begin{proof}
We will prove that the existence of a rank $n$ matrix $X$ such that $\mathrm{im}\,X\not\subset\ker K$ for any \(K\in\mathcal{L}_G\backslash\lbrace 0\rbrace\)
is equivalent to the existence of an $m\times(m-n)$ matrix $V$ such that $VAV^t\notin\mathcal{L}_G$ for any non-zero $A\in\mathrm{Sym}(m-n)$. The proof is presented only for $\gcr(G)$, since the cases of $\wmlt(G)$ and $\mlt(G)$ are analogous after restriction to PSD matrices.

Let $X$ be a matrix that satisfies the conditions of \Cref{th:Sinn2}(1), that is, if $\mathrm{im}\,X\subset\ker K$, then \(K\notin\mathcal{L}_G\backslash\lbrace 0\rbrace\). Equivalently, the inclusion can be formulated as $(\ker K)^\perp\subset(\mathrm{im}\,X)^\perp$. Since $K$ is symmetric, $\ker K=(\mathrm{im}\,K)^\perp$.
Taking $V$ such that $\mathrm{im}\,V=(\mathrm{im}\,X)^\perp$ yields 
$\mathrm{im}\,K\subset\mathrm{im}\,V$.
Note that for any $K\in\mathrm{Sym}(m)$, $\mathrm{im}\,K\subset\mathrm{im}\,V$ is equivalent to $K=VAV^t$ for some $A\in\mathrm{Sym}(m-n)$. By assumption, \(K\notin\mathcal{L}_G\backslash\lbrace 0\rbrace\), hence $VAV^t\notin\mathcal{L}_G$ for any $A\neq 0$.

Conversely, let $V$ satisfy the conditions of \Cref{th:alt}(1). It can be similarly checked that any matrix $X$ such that $\mathrm{im}\,X=(\mathrm{im}\,V)^\perp$ satisfies \Cref{th:Sinn2}(1).
\end{proof}

An immediate consequence of \Cref{th:Sinn2} is that $\gcr(G)$ gives an upper bound to $\mlt(G)$. This result was originally stated in {\cite{Uhler}} as the Elimination Criterion (Theorem 3.3) and proved with different techniques (see \Cref{subsec:eliminationIdeals}). 

\begin{corollary}\label{cor:upper_bound}
\(\mlt(G)\leq\gcr(G)\).
\end{corollary}

A direct consequence of the \(\gcr(G)\) characterization of (\ref{box:all}) is that it can be computed in random polynomial time by checking the rank of the Jacobian of $\pi_G$ at a random point. For more details, see \cite[Algorithm 3.2]{Bernstein2020GCR}, which was originally designed for uncolored bipartite graphs. 
However, \(\gcr(G)\) is not always a sharp bound for \(\mlt(G)\). The smallest uncolored graph where the maximum likelihood threshold and the generic completion rank do not coincide is the bipartite graph \(K_{5,5}\) has \(\mlt(K_{5,5})=4\) but \(\gcr(K_{5,5})=5\), see \cite{blekherman2019maximum,Bernstein2022MLT}. In the colored case, the 4-cycle with equally colored vertices in \Cref{fig:cycles} already displays this behavior, as we prove in \Cref{prop:mCycles}.

\subsection{Elimination ideals of a graph}\label{subsec:eliminationIdeals}

The rank approach to compute the generic completion rank provides an upper bound for the maximum likelihood threshold, but whenever this bound is not sharp, it does not provide further insight into its exact value. However, the original proof of this upper bound \cite[Theorem 3.3]{Uhler} provides a constructive definition of the generic completion rank via the computation of elimination ideals. 

\begin{definition}
Let the \emph{$n$-th elimination ideal of $G$}, \(I_{G,n}\), be the vanishing ideal of the projection of $\mathrm{Sym}(m,n)$ via $\pi_G$, that is,
\[I_{G,n}:=\left(I_{n+1}(S)+\left\langle \left\lbrace t_i-\langle S,K_i\rangle\right\rbrace_{1\leq i\leq d}\right\rangle\right)\cap\mathbb{R}[t_1,\dots,t_d],\]
\noindent
where \(I_{n+1}(S)\) denotes the ideal of \((n+1)\)-minors of the \(m\times m\) symmetric matrix \(S=(s_{ij})\) and \(K_1,\dots,K_d\) is a basis of \(\mathcal{L}_G\).
\end{definition}

The $n$-th elimination ideal of $G$ consists of all polynomials in \(\mathbb{R}[t_1,\dots,t_d]\) that vanish for all sufficient statistics \(t_i=\langle S,K_i\rangle\) obtained from a rank $n$ sample covariance matrix $S=(s_{ij})$. More precisely, $I_{G,n}$ is the vanishing ideal of $\pi_G(\mathrm{Sym}(m,n))$. Thus, while \Cref{def:gcr} characterizes the generic completion rank geometrically via the dimension of the image under the projection map $\pi_G$, this is the algebraic characterization originally proposed in \cite{Uhler}:

\begin{proposition}
\(\gcr(G)\) is the minimum $n$ for which \(I_{G,n}=(0)\). 
\end{proposition}

\begin{example}
Consider the colored 2-cycle \(C_2\) from \Cref{ex:cycles,ex:suffStats}. Then
\[I_{C_2,1}=\langle s_{11}s_{22}-s_{12}^2,t_1-(s_{11}+s_{2}),t_2-2s_{12}\rangle\cap\mathbb{R}[t_1,t_2]=(0),\]
\noindent
thus \(\mlt(C_2)\leq\gcr(C_2)=1\) and the MLE exists for rank one matrices \(S\) with probability 1.
\end{example}

If \(I_{G,n}\neq (0)\), the MLE can either not exist, exist with probability strictly between 0 and 1, or exist with probability 1, as illustrated in \Cref{fig:elimination_ideals}. A significant advantage of knowing the elimination ideal - beyond mere knowledge of \(\gcr(G)\) - is that now we can study the generators of \(I_{G,n}\) to obtain bounds for maximum likelihood thresholds.

\begin{figure}[h]
\begin{tikzpicture}[x=0.75pt,y=0.75pt,yscale=-0.75,xscale=0.75]
%uncomment if require: \path (0,182); %set diagram left start at 0, and has height of 182

%Curve Lines [id:da026708421021965578] 
\draw [fill={rgb, 255:red, 245; green, 166; blue, 35 }  ,fill opacity=1 ]   (227,87) .. controls (13,134) and (13,134) .. (204,25) ;
%Curve Lines [id:da42078786055828066] 
\draw [fill={rgb, 255:red, 245; green, 166; blue, 35 }  ,fill opacity=1 ]   (410,88) .. controls (196,135) and (196,135) .. (387,26) ;
%Curve Lines [id:da1319191354882594] 
\draw [fill={rgb, 255:red, 245; green, 166; blue, 35 }  ,fill opacity=1 ]   (591,87) .. controls (377,134) and (377,134) .. (568,25) ;
%Curve Lines [id:da18416450953508123] 
\draw [color={rgb, 255:red, 74; green, 144; blue, 226 }  ,draw opacity=1 ][line width=2.25]    (312,69) .. controls (358,43) and (364,103) .. (404,73) ;
%Curve Lines [id:da6594213693420136] 
\draw [color={rgb, 255:red, 74; green, 144; blue, 226 }  ,draw opacity=1 ][line width=2.25]    (38,157) .. controls (40,98) and (164,55) .. (204,25) ;
%Curve Lines [id:da9969142042123911] 
\draw [color={rgb, 255:red, 74; green, 144; blue, 226 }  ,draw opacity=1 ][line width=2.25]    (403,155) .. controls (401,89) and (556,92) .. (576,48) ;
%Curve Lines [id:da04379537016156965] 
\draw [color={rgb, 255:red, 74; green, 144; blue, 226 }  ,draw opacity=1 ][line width=2.25]    (225,157) .. controls (242,99) and (272,95) .. (312,69) ;
\end{tikzpicture}
\caption{In orange, the interior of the cone of sufficient statistics $\mathcal{C}_G$. In blue, the generator $f$ of a principal elimination ideal $I_{G,n}=\langle f\rangle$. From left to right, the MLE does not exist, it exists with probability strictly between 0 and 1, and it exists with probability 1.}
\label{fig:elimination_ideals}
\end{figure}
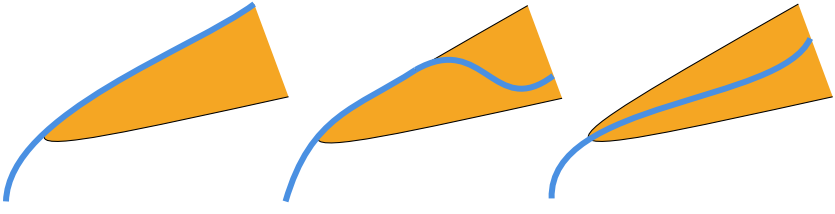

\begin{lemma}
    If there exists \(f(t_1,\dots,t_n)\in I_{G,n}\) such that \(f(\langle \Sigma,K_1\rangle,\dots,\langle \Sigma,K_d\rangle)\neq 0\) for any \(\Sigma=(\sigma_{ij})\) in \(\mathrm{PD}(m)\), then $\wmlt(G)\geq n+1$.
\end{lemma}

\begin{proof}
By construction, \(f(\langle S,K_1\rangle,\dots,\langle S,K_d\rangle)= 0\) for all \(S\in\mathrm{Sym}(m,n)\). But then $\pi_G(\Sigma)\neq\pi_G(S)$ for any \(\Sigma\in\mathrm{PD}(m)\), hence the MLE never exists for rank $n$.
\end{proof}

We design a positivity certificate for (at least one of) the generators of \(I_{G,n}\) when evaluated at positive definite matrices. In particular, we consider the Cholesky decomposition \(\Sigma=LL^t\), where $L=(l_{ij})$ is a lower-triangular matrix with strictly positive diagonal entries, and compute a sum of squares (SOS) decomposition of $f$ in $\mathbb{R}[l_{ij}]$. 

\begin{algorithm}[H]
    \caption{Non-negativity certificate for generators \(I_{G,n}\) via sums-of-squares (SOS)}\label{alg:SOS}
    \KwData{ $m$ - number of vertices of the graph; \(K_1,\dots,K_d\) standard basis of \(\mathcal{L}_G\)\;
    \(f_1,\dots,f_k\in\mathbb{R}[t_1,\dots,t_d]\) - generators of \(I_{G,n}\).}
    \KwResult{for each $f_i$, whether it changes sign or not for PD matrices.}
    \Begin{
Generate random $m\times m$ positive definite matrices $\Sigma$\;
For each $i$, compute \(f_i(\Sigma):=f_i(\langle \Sigma,K_1\rangle,\dots,\langle \Sigma,K_d\rangle)\)\;
   \If{$\Sigma$ and $\Sigma'$ are found such that $\mathrm{sign}(f_i(\Sigma)f_i(\Sigma'))<0$}{\KwRet{$f_i$ changes signs for PD matrices}\;}
   \Else{Define an $m\times m$ lower triangular matrix $L=(l_{ij})$\;
         Compute a sum-of-squares decomposition of \(f_i(LL^t)\)\;
         }
    \KwRet{output of the SOS problem for \(f_i(LL^t)\)}
    }
\end{algorithm}

\begin{remark}
Note that an SOS decomposition only provides a non-negativity certificate. Therefore, \Cref{alg:SOS} requires further analysis of its output to ensure that the resulting SOS is strictly positive by using the fact that the diagonal entries $l_{ii}$ of the lower triangular matrix $L=(l_{ij})$ are strictly positive. 
\end{remark}

\begin{example}\label{ex:G9}
Let $G$ be the colored 4-cycle $G_9$ in \cite[Table 2]{Uhler} displayed in \Cref{table: all graphs}. Consider the sufficient statistics $t_1=s_{1,1}$, $t_2=s_{2,2}+s_{4,4}$, $t_3=s_{3,3}$, $t_4=2s_{1,2}+2s_{1,4}$, $t_5=2s_{2,3}$, $t_6=2s_{3,4}$. The elimination ideals are $I_{G,2}=I_{G,3}=0,$ and
\begin{equation}\label{eq:graph9_IG1}
  I_{G,1}=\langle t_3t_4^2-t_1t_5^2-2t_1t_5t_6-t_1t_6^2,4t_2t_3-t_5^2-t_6^2\rangle.  
\end{equation}
\noindent
Let $f_1$ and $f_2$ be the generators of $I_{G,1}$. An empirical test shows the following signs: 

    \begin{table}[h]
    \centering
    \begin{tabular}{c|cccc}
    n & 1 & 2 & 3 & 4\\
    \hline
    $f_1$ & 0 & $\pm$ & $\pm$ & $\pm$\\
    $f_2$ & 0 & + & + & +
    \end{tabular}
\end{table}

\noindent
The generator $f_2$ evaluated at $\Sigma=LL^t$, where $L=(l_{ij})$ is lower-triangular with $l_{ii}>0$, is
\begin{multline*}
f_2(\Sigma)=4(-l_{13}l_{22}+l_{12}l_{23})^2+4(l_{12}l_{33})^2+4(-l_{14}l_{23}+l_{13}l_{24})^2\\
+4(-l_{14}l_{33}+l_{13}l_{34})^2+4(l_{13}l_{44})^2+4(l_{22}l_{33})^2\\
+4(-l_{24}l_{33}+l_{23}l_{34})^2+4(l_{23}l_{44})^2+4(l_{33}l_{44})^2>0.
\end{multline*}
 
\noindent
Thus, the MLE does not exist for 1 observation and $\wmlt(G)=\mlt(G)=\gcr(G)=2$.
\end{example}

\begin{example}\label{ex:G11}

Consider the colored 4-cycle $G_{11}$ in \cite[Table 2]{Uhler} displayed in \Cref{table: all graphs}.
Consider the sufficient statistics $t_1=s_{1,1}+s_{3,3}$, $t_2=s_{2,2}+s_{4,4}$, $t_3=2s_{1,2}$, $t_4=2s_{1,4}+2s_{2,3}$, $t_5=2s_{3,4}$. The elimination ideals are $I_{G,2}=I_{G,3}=0$, and 
$I_{G,1}=\langle 4t_1t_2-t_3^2-t_4^2+2t_3t_4-t_5^2\rangle$.
\noindent
For any positive definite matrix $\Sigma=LL^t$, the single generator of $I_{G,1}$ is
$$f(\Sigma)=l_{33}^2l_{44}^2+l_{22}^2l_{33}^2+l_{11}^2l_{44}^2+g,$$ 

\noindent
where $g$ is an SOS. Therefore, $\wmlt(G)=\mlt(G)=\gcr(G)=2$.
\end{example}

We use \Cref{alg:SOS} to provide detailed computations for all colored 3-cycles in \Cref{table:3cycles} in \Cref{app:SM}.
Moreover, we are now able to complete the gaps in \cite[Table 2]{Uhler} for colored 4-cycles.

\begin{proposition}\label{proposition: resolved conjectures}
For graphs 9, 11, 14, 17 in \cite[Table 2]{Uhler}, \(\gcr(G)=\mlt(G)=\wmlt(G)=2\). 
\end{proposition}
\begin{proof}
The statement is proven for graphs 9 and 11 in \Cref{ex:G9} and \Cref{ex:G11}, respectively, using \Cref{alg:SOS}. Note that if $\mathcal{L}_G\subset\mathcal{L}_{G'}$, then $\mlt(G)\leq\mlt(G')$ and $\wmlt(G)\leq\wmlt(G')$. Since $\mathcal{L}_{G_9}\subset\mathcal{L}_{G_{17}}$ and $\mathcal{L}_{G_{11}}\subset\mathcal{L}_{G_{14}}$, the statement holds.
\end{proof}

\subsection{Generalized matrix completion problems}\label{subsec:CompletionProblems}
The study of the existence of maximum likelihood estimators for uncolored Gaussian graphical models was initiated by Dempster \cite{dempster1972covariance} through the lens of PD matrix completion. In an uncolored graph, missing edges correspond to unknown entries in the sample covariance matrix. The MLE exists if and only if there is a PD completion for this partial matrix. Consequently, the MLT represents the minimum rank of a sample covariance matrix such that a PD completion exists almost surely. Similarly, the GCR can be interpreted as a low-rank completion problem, that is, the lowest rank of a completion with (possibly) complex entries of a generic partial matrix. 

\begin{remark} As in \Cref{subsec:gcr}, not all properties hold over the reals. When we only consider real completions, the GCR is the lowest completion rank for a full-dimensional semi-algebraic set of matrices with a given pattern of unknown entries. However, this is not necessarily true for a generic partial matrix: there can be other full-dimensional semi-algebraic sets with higher lowest completion rank (so-called typical ranks). See \cite{Bernstein2020GCR} for more details on generic and typical ranks.
\end{remark}

Introducing a coloring function $\lambda$ generalizes the notion of matrix completion. Rather than merely filling in gaps, a coloring imposes symmetry constraints that allow entries belonging to the same color class to be modified, provided their sum —-- the sufficient statistic —-- remains preserved. In this framework, we can think of elements in $\mathrm{fiber}_G(S)$ as \emph{generalized matrix completions} of the sufficient statistics $\pi_G(S)$.

\begin{example}\label{ex:c4fiber}
A matrix completion of \(S=(s_{ij})\) associated to the uncolored 4-cycle $G$ is of the form displayed below for some $x,y\in\mathbb{R}$. If we color all vertices equally as in $C_4$ in \Cref{ex:c4proj}, then the generalized matrix completion of \(S=(s_{ij})\) is required to preserve the trace but not the individual diagonal entries.
$$\mathrm{fiber}_G(S)=\left\lbrace\begin{pmatrix}
 s_{11} & s_{12} & x & s_{14} \\
 s_{12} & s_{22} & s_{23} & y \\
 x & s_{23} & s_{33} & s_{34} \\
 s_{14} & y & s_{34} & s_{44}
\end{pmatrix}:x,y\in\mathbb{R}\right\rbrace$$
$$\mathrm{fiber}_{C_4}(S)=\left\lbrace\begin{pmatrix}
 s_{11}+\varepsilon_1 & s_{12} & x & s_{14} \\
 s_{12} & s_{22}+\varepsilon_2 & s_{23} & y \\
 x & s_{23} & s_{33}+\varepsilon_3 & s_{34} \\
 s_{14} & y & s_{34} & s_{44}+\varepsilon_4
\end{pmatrix}:x,y,\varepsilon_i\in\mathbb{R},\,\sum_{i=1}^4\varepsilon_i=0\right\rbrace$$
\end{example}

In \cite{blekherman2019maximum}, Blekherman and Sinn rephrase the characterizations of MLE and GCR of uncolored Gaussian graphical models as matrix completion problems through the lens of algebraic geometry \cite[Proposition 1.2 (1), Proposition 1.6 (1)]{blekherman2019maximum}. Their framework can be translated into the colored setting with minor variations.
Let $\mathcal{L}_G^\perp$ be the orthogonal space to $\mathcal{L}_G$ in $\mathrm{Sym}(m)$. Note that $\ker\pi_G=\mathcal{L}_G^\perp$,
hence matrices in $S+\mathcal{L}_G^\perp$ are (generalized) matrix completions of $S$.

\begin{definition}\label{def:IG}
Consider a basis $K_{d+1},\dots,K_r$ of $\mathcal{L}_G$, with $r=m(m+1)/2$. We call $I_G$ the ideal in $\mathbb{R}[y_1,\dots,y_m]$ generated by quadratic forms $(y_1\,\dots\,y_m)K_i(y_1\,\dots\,y_m)^t$, for $d< i\leq r$.
\end{definition}

Recalling the one-to-one correspondence between symmetric matrices and quadratic forms, $\mathcal{L}_G^\perp$ can be identified with the linear space of quadratic forms $\langle I_G\rangle_2$, i.e. the degree 2 part of the homogeneous ideal $I_G$. Now consider the quotient ring $R=\mathbb{R}[y_1,\dots,y_m]/I_G$. Let $R_i$ denote the homogeneous degree $i$ part of $R$. 
Note that $R_1=\mathbb{R}[y_1,\dots,y_m]_1$, so we can view an observation $X=(x_1,\dots,x_m)^t$ in $\mathbb{R}^m$ as a polynomial (or a linear form) \(x_1y_1+\dots+x_my_m\) in \(R_1\).
Moreover, $R_2$ can be identified with $\mathbb{S}^m/\mathcal{L}_G^\perp$.  That is, $R_2$ is the $d$-dimensional vector space of sufficient statistics of all symmetric matrices and the cone of sufficient statistics $\mathcal{C}_G$ can be regarded as the subspace of $R_2$ consisting of the sufficient statistics of all positive definite matrices. 

\begin{remark}\label{rem:monomial}
For an uncolored graph $G$, $I_G$ is a squarefree monomial ideal. Adding a coloring breaks both properties, but we still have an ideal generated in degree 2. Even when $I_G$ is not monomial, one can take a monomial basis of $R_2$ by choosing certain vertices and edges to represent each color class. See \Cref{ex:IG} for details on this construction.
\end{remark}

\begin{example}\label{ex:IG} Consider the 4-cycle with equally colored vertices and a basis of $\mathcal{L}_G^\perp$ given by 
$$\begin{bmatrix}
        1 & 0 & 0 & 0\\
         0 & -1 & 0 & 0\\
        0 & 0 & 0& 0\\
         0 & 0 & 0 & 0
    \end{bmatrix}, 
   \begin{bmatrix}
        1 & 0 & 0 & 0\\
         0 & 0 & 0 & 0\\
        0 & 0 & -1& 0\\
         0 & 0 & 0 & 0
    \end{bmatrix},\begin{bmatrix}
        1 & 0 & 0 & 0\\
         0 & 0 & 0 & 0\\
        0 & 0 & 0& 0\\
         0 & 0 & 0 & -1
    \end{bmatrix},\begin{bmatrix}
        0 & 0 & 1 & 0\\
         0 & 0 & 0 & 0\\
        1 & 0 & 0& 0\\
         0 & 0 & 0 & 0
    \end{bmatrix},\begin{bmatrix}
        0 & 0 & 0 & 0\\
         0 & 0 & 0 & 1\\
        0 & 0 & 0& 0\\
         0 & 1 & 0 & 0
    \end{bmatrix}.
    $$

Then $I_G=\langle y_1y_3,y_2y_4,y_1^2-y_2^2,y_1^2-y_3^2,y_1^2-y_4^2\rangle\subset \mathbb{R}[y_1,y_2,y_3,y_4]$, and the first and second degree parts of $R=\mathbb{R}[y_1,y_2,y_3,y_4]/I_G$ are $R_1=\langle y_1,y_2,y_3,y_4\rangle_\mathbb{R}$ and $R_2=\langle y_1^2, y_1y_2, y_2y_3, y_3y_4, y_1y_4\rangle_\mathbb{R}$, considering the monomial basis obtained from representatives of each color class. 
\end{example}

We now adapt and extend the characterization of thresholds in terms of matrix completion problems given by Blekherman-Sinn in \cite[Proposition 1.2 (1), Proposition 1.6 (1)]{blekherman2019maximum}. We interpret their linear series as a collection of observations  \(X_1,\dots,X_n\in\mathbb{R}^m\) and denote by \(\langle X_1,\dots,X_n\rangle_2\) their degree two span when considered as elements of $R_1$. 

\begin{theorem}\label{th:Sinn1} Let $G$ be a colored graph with vertex set $V=[m]$. Then
    \begin{enumerate}
        \item \(\gcr(G)\) is the smallest \(n\) for which 
        the vector space \(\langle X_1,\dots,X_n\rangle_2+\mathcal{L}_G^{\perp}=\mathrm{Sym}(m)\)
        for some observations \(X_1,\dots,X_n\in\mathbb{R}^m\). 
        \item \(\mlt(G)\) is the smallest \(n\) for which the vector space \(\langle X_1,\dots,X_n\rangle_2+\mathcal{L}_G^{\perp}\) contains a positive definite matrix for generic observations \(X_1,\dots,X_n\in\mathbb{R}^m\).
        \item \(\wmlt(G)\) is the smallest \(n\) for which the vector space \(\langle X_1,\dots,X_n\rangle_2+\mathcal{L}_G^{\perp}\) contains a positive definite matrix for some observations \(X_1,\dots,X_n\in\mathbb{R}^m\).
    \end{enumerate}
\end{theorem}
\begin{proof}
The proof follows exactly as in the uncolored setting in \cite{blekherman2019maximum}. The key fact is the equality \(\langle X_1,\dots,X_n\rangle_2=T_SV_n\), where $X$ is an $m\times n$ matrix with $X_i$ as column $i$ and $S=XX^t$. Note that \(\mathrm{fiber}_G(S)=S+\mathcal{L}_G^\perp\). One can check that the existence of a PD matrix in the fiber is equivalent to the existence of such a matrix in $T_SV_n+\mathcal{L}_G^\perp$.
\end{proof}

\begin{remark} 
Although the proof of \Cref{th:Sinn1} does generalize directly from \cite[Proposition 1.2, 1.6]{blekherman2019maximum}, results by Blekherman-Sinn on clique-sums of uncolored graphs \cite[Theorems 1.13, 1.15]{blekherman2019maximum} do not extend to the colored setting.
If $G_1$ and $G_2$ are uncolored graphs and  $G=G_1\boxtimes G_2$ is their clique-sum, then $\mlt(G)=\max\lbrace\mlt(G_1),\mlt(G_2)\rbrace$ and $\gcr(G)=\max\lbrace\gcr(G_1),\gcr(G_2)\rbrace$. One can naturally extend the clique-sum definition \cite[Definition 1.9]{blekherman2019maximum} to colored graphs by requiring that the gluing is done along cliques with compatible coloring.
However, in the colored case, ideals $I_G$ are no longer monomial and squarefree, see \Cref{rem:monomial}, which results in the failure to interpret $\mathcal{V}(I_G)$ as a subspace arrangement and consequently the failure to interpret clique sums as linear joins. \end{remark}

\begin{example}
Consider colored 3-cycles $G_1$ and $G_2$ below and let $G=G_1\boxtimes G_2$ be their clique-sum along the 2-clique corresponding to the red vertices.

\begin{center}
\begin{tikzpicture}[
    blue node/.style={circle, fill=black, inner sep=0pt, minimum size=6pt},
    highlight node/.style={circle, fill=black, draw=red!80, ultra thick, inner sep=0pt, minimum size=7pt},
    blue line/.style={thick, black}
]

% --- FIRST GRAPH (LEFT) ---
  \node[blue node] (L_left) at (-3, 0) {};
  \node[highlight node] (L_top) at (-1.5, 1.0) {};
  \node[highlight node] (L_bot) at (-1.5, -1.0) {};

  \draw[blue line] (L_left) -- (L_top);
  \draw[blue line] (L_left) -- (L_bot);
  \draw[blue line] (L_top) -- (L_bot);
    % Graph Label G1
  \node at (-2.25, -1.6) {$G_1$};
  % --- OPERATOR (*) ---
  \node at (-0.5, 0) {\Large $\ast$};

  % --- SECOND GRAPH (MIDDLE) ---
  \node[highlight node] (M_top) at (0.5, 1.0) {};
  \node[highlight node] (M_bot) at (0.5, -1.0) {};
  \node[blue node] (M_right) at (2, 0) {};

  \draw[blue line] (M_top) -- (M_bot);
  \draw[blue line] (M_top) -- (M_right);
  \draw[blue line] (M_bot) -- (M_right);
% Graph Label G2
  \node at (1.25, -1.6) {$G_2$};
  % --- EQUALS SIGN (=) ---
  \node at (3.1, 0) {\Large $=$};

  % --- RESULT GRAPH (RIGHT) ---
  \node[blue node] (R_left) at (4.3, 0) {};
  \node[highlight node] (R_top) at (5.5, 1.0) {};
  \node[highlight node] (R_bot) at (5.5, -1.0) {};
  \node[blue node] (R_right) at (6.7, 0) {};

  \draw[blue line] (R_left) -- (R_top);
  \draw[blue line] (R_left) -- (R_bot);
  \draw[blue line] (R_top) -- (R_bot);
  \draw[blue line] (R_top) -- (R_right);
  \draw[blue line] (R_bot) -- (R_right);
% Graph Label G
  \node at (5.5, -1.6) {$G=G_1\boxtimes G_2$};
\end{tikzpicture}
\end{center}
\noindent
One can check that $\mlt(G_i)=\gcr(G_i)=2$, for $i=1,2$, but $\mlt(G)=\gcr(G)=3$.
\end{example}

In the remaining part of this section we showcase how \Cref{th:Sinn1} can be applied to compute maximum likelihood thresholds for specific families of graphs.

\begin{example}\label{ex:th} Let us apply the previous result to determine thresholds for the 4-cycle with equally colored vertices following the construction from \Cref{ex:IG}.
Observations $X_1=(1,0,0,0),X_2=(0,0,1,0)\in\mathbb{R}^4$ are identified with polynomials $y_1,y_3$ in $R_1$ and $\langle y_1,y_3\rangle_2=\langle y_1^2, y_1y_2,y_1y_3,y_1y_4,y_2y_3,y_3^2,y_3y_4\rangle$. Since matrix

$$\begin{bmatrix}
         R_2   \\
         \hline
        y_1^2\\
        y_1y_2\\
        y_2y_3\\
        y_3y_4\\
        y_1y_4
    \end{bmatrix}\begin{bmatrix}
         y_1^2 & y_1y_2 & y_1y_3 & y_1y_4 & y_2y_3 & y_3^2 & y_3y_4\\
         \hline
         1 & 0 & 0 & 0 &  0 & 1 & 0\\
         0 & 1 & 0 & 0 &  0 & 0 & 0\\
         0 & 0 & 0 & 1 &  0 & 0 & 0\\
         0 & 0 & 0 & 0 &  1 & 0 & 0\\
         0 & 0 & 0 & 0 &  0 & 0 & 1\\
    \end{bmatrix}$$
\noindent
is of rank 5, then $\langle y_1,y_3\rangle_2=R_2$. 
Note that $\langle l\rangle_2\neq R_2$ for any $l\in R_1$ because $4=\dim \langle l\rangle_2<5$. Therefore, \(\gcr(G)=2\). 
Moreover, matrices in $\langle X_1\rangle_2 +\mathcal{L}_G^\perp$ are of the form 
$$\begin{bmatrix}
        \lambda+\mu_1+\mu_2+\mu_3 & 0 & \mu_4 & 0\\
         0 & -\mu_1 & 0 & \mu_5\\
         \mu_4 & 0 & -\mu_2& 0\\
         0 & \mu_5 & 0 & -\mu_3
    \end{bmatrix}.$$
Note that for any $\mu_5=\mu4=0$, $\mu_1,\mu_2,\mu_3<0$ and $\lambda+\mu_1+\mu_2+\mu_3>0$, the resulting matrix is positive definite. Therefore, $\wmlt(G)=1$. 
\end{example}

\begin{proposition}
If \(G\) is a colored graph with equally colored vertices, then \(\wmlt(G)=1\).
\end{proposition}
\begin{proof}
Consider observation $X_1=(1,0,\dots,0)\in\mathbb{R}^m$ and describe a positive definite matrix in $\langle X_1\rangle_2+\mathcal{L}_G^\perp$, generalizing the construction in \Cref{ex:th}.
\end{proof}

We finish the section by providing a family of graphs with 
\(\mlt(G)<\gcr(G)\). 

\begin{proposition}\label{prop:mCycles}
If \(G\) is a colored \(m\)-cycle with equally colored vertices. If all edges are colored differently, then \(\gcr(G)=2\).
If \(m\) is even, then \(\mlt(G)=1\).
\end{proposition}

\begin{proof} Let $G$ be a colored graph with all vertices equally colored and all edges differently colored. Note that $d=m+1$ and hence \(m=\dim\langle x\rangle_2<\dim R_2=m+1\) for any \(x\in R_1\). 
Since
\(\langle x_1,x_2\rangle_2=R_2\) for \(x_1=\sum_{i=1}^{\lfloor m/2\rfloor}y_i\) and \(x_2=y_1+\sum_{\lfloor m/2\rfloor}^{m-1}y_i\), then \(\gcr(G)= 2\) by \Cref{th:Sinn1}.

By \Cref{th:Sinn2}, proving \(\mlt(G)=1\) is equivalent to proving that there is no nonzero $K\in\mathcal{L}_G\cap\mathrm{PSD}(m)$ that contains a generic column vector \((a_1 \dots a_m)^t\) in its kernel. Any matrix \(K\) in \(\mathcal{L}_G\cap\mathrm{PSD}(m)\) satisfies \(\lambda_1\geq\vert\lambda_i\vert\). Rewrite the system of equations $K(a_1 \dots a_m)^t=0$ as $A\left(\lambda_1\dots\lambda_{m+1}\right)^t=0$. If $a_1\neq 0$, $\left(\sum_i^m (-1)^{i+1}a_i^2\right)\lambda_1=0$ for $m$ even. Thus $\lambda_1=0$ and $K=0$ is the only positive semidefinite matrix in $\mathcal{L}_G$. 
\end{proof}

If $m$ is odd, then bringing the system $A\left(\lambda_1\dots\lambda_{m+1}\right)^t=0$ in the previous proof into row echelon form yields the equation $\left(a_1^2+\sum_{i=2}^m (-1)^i a_i^2\right)\lambda_1+2a_1a_2\lambda_2=0$. However, computational experiments suggest that the statement still holds for $m$ odd.
\begin{conjecture}
  Let $G$ be a colored $m$-cycle with equally colored vertices. If $m$ is odd, then $\mlt(G)=1$.
\end{conjecture}

\section{Computational topology approaches to ML thresholds}

While symbolic approaches such as elimination ideals or Gröbner bases discussed in Section~\ref{sec:alggeo} give exact criteria, they scale poorly and are computationally infeasible beyond small graphs.  Addressing these limitations requires numerical methods. Typically, ML thresholds would be obtained numerically via optimization algorithms, e.g. by computing the MLE of randomly generated sample covariance matrices of a fixed rank or by maximizing random linear functions over $\mathcal{K}_G\cap\{\mathrm{trace}(K)=1\}$ \cite{UhlerSturmfels}. 
Here, we provide an alternative computational approach and explore $\mathcal{C}_G$ empirically using ideas from TDA on samples from the cone of sufficient statistics.

Specifically, we develop code for sampling within the cone of sufficient statistics and propose to analyze the resulting samples via TDA, with the goal of identifying rank patterns and local connectivity features related to likelihood feasibility. We perform a case study on small graphs, for which symbolic methods still remain feasible and for which the ground truth is known to test the feasibility of our approach. Our experiments, while promising, expose the geometric complexity of colored models and the obstacles that future computational approaches must address.

\subsection{Topological data analysis}
TDA is a novel and active field of research combining concepts from pure and applied mathematics to quantify 'shape' in data \cite{Carlsson2009,Ghrist2008,edelsbrunner2002topological,carlsson2005,robins1999towards}. The most prominent method in TDA is persistent homology \cite{Carlsson2009,Ghrist2008,edelsbrunner2002topological,carlsson2005,robins1999towards}, an automatic, robust, and interpretable technique that studies topological invariants across multiple scales. 
Persistent homology applied to point cloud data yields persistence barcodes that capture connectivity patterns in dimension $k$ such as connected components ($k=0$), loops ($k=1$), and voids ($k=2$). When computed on annular neighbourhoods, persistent homology can detect points near intersections or boundaries in non-manifold data \cite{Stolz2020}. The resulting local persistence barcodes can be interpreted to reflect qualitative changes in the local geometry across rank strata. Here, we leverage this property to investigate whether projections of positive semi-definite matrices onto their sufficient statistics lie close to the boundary of the cone of sufficient statistics of positive definite matrices.

\subsubsection{Persistent homology}\label{section: local TDA}
Typical input data to persistent homology is a collection of points in $\mathbb{R}^n$. This data is then used to construct a filtered simplicial complex that includes data points as its vertices. The most common choice of filtration is the Vietoris-Rips filtration \cite{Vietoris1927}.
\begin{definition}[Vietoris-Rips filtration $\VR_P^{\bullet}$]
\label{def:VR complex}
Let $P$ be a point cloud, $d$ be a distance function on $P$, and $\varepsilon \in \mathbb{R}^+_0$. The \emph{Vietoris-Rips complex at parameter $\varepsilon$}, denoted $\VR_P^{\varepsilon}$, is a simplicial complex that has the points in $P$ as its vertex set and includes the $n$-simplex $\tau = (p_0, \dots, p_n)$ if $d(p_i, p_j) \leq \varepsilon$ for all $i,j \in \{1,\dots,n\}$ with $i \neq j$. A \emph{Vietoris-Rips filtration} $\VR_P^{\bullet}$ is a nested sequence of simplicial complexes $\VR_P^{\varepsilon}$ obtained by considering a sequence of increasing $\varepsilon$.
\end{definition}

For a given filtration and a fixed field $\mathbb{F}$, the $k$-dimensional persistence module is the sequence of $\mathbb{F}$-vector spaces together with linear maps induced by inclusions in the filtration:

\begin{definition}[k-dimensional persistence module $PH_k(X^\bullet)$]
\label{def:k-PH}
For a filtered simplicial complex $$T^\bullet =T^1 \xhookrightarrow{\iota^1} T^2 \xhookrightarrow{\iota^2} \cdots \xhookrightarrow{\iota^{N-2}} T^{N-1} \xhookrightarrow{\iota^{N-1}} T^N,$$ the \emph{k-dimensional persistence module} is defined by $$PH_k(T^\bullet) =  H_k(T^1; \mathbb{F}) \xrightarrow{\phi^1} H_k(T^2; \mathbb{F}) \xrightarrow{\phi^2} \cdots \xrightarrow{\phi^{N-1}}  H_k(T^N; \mathbb{F}),$$ where $\phi^{\varepsilon}$ are the maps induced by $\iota^{\varepsilon}$ and $H_k(X^i; \mathbb{F})$ is the $k$-th homology group of $T^i$ over $\mathbb{F}$.
\end{definition}
Here, we work with $\mathbb{F} = \mathbb{Z}/2\mathbb{Z}$.
The structure theorem of persistent homology provides a decomposition of the \emph{k-dimensional persistence module} into a unique persistence barcode \cite{carlsson2005} which is stable to small perturbations of the points in $P$ \cite{Chazal2009,cohen2005stability}. \emph{Persistent homology} refers to the process of constructing a persistence module from data, whose structure is then typically summarised by its persistence barcode.

\subsubsection{Understanding existence of MLE using persistent homology of annular neighbourhoods}\label{sec:MLEviaTDA}

For finite samples from intersecting manifolds, persistent homology computed on annular neighbourhoods can detect singular regions even when none of the data points were sampled precisely from these singularities \cite{Stolz2020}. The strategy relies on the fact that for $n$-dimensional spheres, the dimension of the $i$-th homology group is equal to 1 if $i = n$ and 0 otherwise. For fixed real parameters $0<r<s$, we compute the annular neighbourhood $A_x$ of each point $x \in P$, i.e. all points $y \in P$ such that $r\leq d(x,y) \leq s$. We then compute the $k$-th persistent homology of the Vietoris-Rips filtration $\VR_{A_x}^{\bullet}$. If the barcode does not approximate a single $k$-sphere, the point $x$ lies close to a singular region.

We propose Algorithm~\ref{alg:MLT estimation for large graphs} for MLT estimation using persistent homology on annular neighbourhoods and perform a case study on $G_3$, $G_6$ and $G_{18}$ from Table~\ref{table: all graphs}. For a fixed sample size $t$, we set $R_2=1.3\times \text{median distance among points in $D$}$ and $R_1=7R_2/8$, which have shown the best performance in our experiments. We exclude points
    that have less than 4 neighbours in their annular neighbourhoods. For $G_3$ we restricted ourselves to computing persistent homology in dimension 1 since after the compactification with the trace constraint, the ambient space has dimension 2. In $G_6$, we used both dimension 1 and 2  persistent homology as the cone of sufficient statistics is 4-dimensional, but can be further reduced to 3 dimensions as for $G_3$. For $G_{18}$, where the cone of sufficient statistics is high-dimensional, we restrict computations to 2-dimensional homology due to the sampling limitations discussed below. To retain computational efficiency, we do not compute persistent homology beyond dimension  2.

    Intuitively, to check whether MLE exists for a matrix $S \in \mathrm{PSD}(m,n)$, with $n<m$, we approximate the interior of the cone of sufficient statistics $\pi_G(\mathcal{C}_G)$ by uniformly sampling a set $\mathcal{S}_{\mathrm{PD}}\subseteq \mathrm{PD}(m)$ (following \cite{mittelbach2012sampling}). We then compute the persistent homology of an annular neighbourhood of $\pi_G(S)$ in $\pi_G(\mathcal{S}_{\mathrm{PD}})$. The annular neighbourhood thereby gives us a discrete proxy of the boundary of the neighbourhood of $\pi_G(S)$ in the cone of sufficient statistics.  If the  barcode resulting from the persistent homology calculation reveals an annular neighbourhood that resembling $k$-sphere, then we conclude that $\pi_G(S)$ lies in the interior of the cone of sufficient statistics (see Algorithm~\ref{alg:local homology of sufficient stats}). 
    
\begin{algorithm}[H]
    \caption{MLT estimation using persistent homology}\label{alg:MLT estimation for large graphs}
    \KwData{$G$ - undirected graph, $n$ - 
    candidate threshold,
    $N_n$ - number of $n$-tuples of observations,  
    %$a$ - trace of the Gram matrix,
    $M$ - number of samples from $\mathrm{PD}(m)$ with trace 1}
    %$a$}
    \KwResult{$N_{n,\mathrm{int}}$ - the number of $n$-tuples of observations whose sufficient statistics lie in the interior of $\mathcal{C}_G$}
    \Begin{
$N_{n,\mathrm{int}}\longleftarrow 0$\;
$D\longleftarrow \varnothing$\;
Compute $\mathcal{S}_{\mathrm{PD}} \subseteq \mathrm{PD}(m)$ - the i.i.d. set of positive definite matrices with trace $1$ sampled uniformly~\cite{mittelbach2012sampling}\;
    
    Uniformly sample $\mathcal{X}=\{X^{(1)},\dots,X^{(N_n)}\}$, where $X^{(i)}=\{X^{(i,1)},\dots,X^{(i,n)}\}\subseteq [-1,1]^{m}$\;
    \For{$X^{(i)}\in \mathcal{X}$} {
    $S^{(i)}=X^{(i)}(X^{(i)})^T$\;
    $S^{(i)}=\frac{1}{\operatorname{trace}(S^{(i)})}S^{(i)}$\;
    $D\longleftarrow D \cup \{\pi_G(S^{(i)})\}$\;}
    Compute $D_\text{int}$ with \Cref{alg:local homology of sufficient stats}\;
    $N_{n,\mathrm{int}}\longleftarrow |D_\text{int}|$\;

    \KwRet{$N_{t,\mathrm{int}}$}
    }
\end{algorithm}

\begin{algorithm}[H]
    \caption{Persistent homology of annular regions of sufficient statistics}\label{alg:local homology of sufficient stats}

\KwData{Finite point set $D \subset \mathbb{R}^n$, real parameters $0 < R_1 < R_2$}
    
\KwResult{Partition of $D$ into boundary points $D_\text{bdry}$ and interior points $D_\text{int}$}  
    
\Begin{ $D_\text{int}\longleftarrow \varnothing$, $D_\text{bdry}\longleftarrow \varnothing$\; 
\For{$y \in D$}{
Find ${A}_{y} = \{ x \in D \text{ satisfying } R_1 \leq \|x-y\| \leq R_2\}$\;
\If{$|{A}_{y}| \geq 100$}{
Choose 100 points from ${A}_{y}$ using the sequential maxmin algorithm\;
Compute $PH_{1}(A_y)$, the $1$-dim barcode of $A_y$ and $PH_{2}(A_y)$, the $2$-dim barcode of $A_y$\;
Calculate $N_{1,y} = \#\{[a,b) \in PH_{1}(A_y) \text{ with }(b-a) > (R_2 - R_1)\}$ and $N_{2,y} = \#\{[a,b) \in PH_{2}(A_y) \text{ with }(b-a) > (R_2 - R_1)\}$\;
\uIf{$N_{2,y} \geq 1$}{
$D_\text{int}\longleftarrow D_\text{int}\cup \{y\}$
} 
\uElseIf{$N_{1,y} \geq 1$}{$D_\text{int}\longleftarrow D_\text{int}\cup \{y\}$}
\uElse{$D_\text{bdry}\longleftarrow D_\text{bdry}\cup \{y\}$}
}
}
\KwRet{$(D_\text{int},D_\text{bdry})$}
 }   
\end{algorithm}

For $G_{18}$ (see Figure \ref{fig:combined-cycles}), we find that our approach reports the most accurate results with the percentage of boundary points being almost 1 for one observation (MLE does not exist), below 0.3 for two observations (MLE exists with probability $p \in (0,1)$), and below 0.1 for 3 observations (MLE exists with $p=1$). This corresponds to $\wmlt(G_{18})=2,\mlt(G_{18})=3$. For $G_3$ and $G_6$, where the MLE exists with $p=1$ for 1,2, or 3 observations, we find that the reported percentage of boundary points remains below 0.2 and 0.4, respectively, across all sample sizes, corresponding to $\wmlt(G_3)=\mlt(G_3)=1$ and $\wmlt(G_6)=\mlt(G_6)=1$ respectively. Overall, our approach appears to successfully capture relative trends but does not provide sufficiently sharp decision boundaries to draw definitive conclusions.
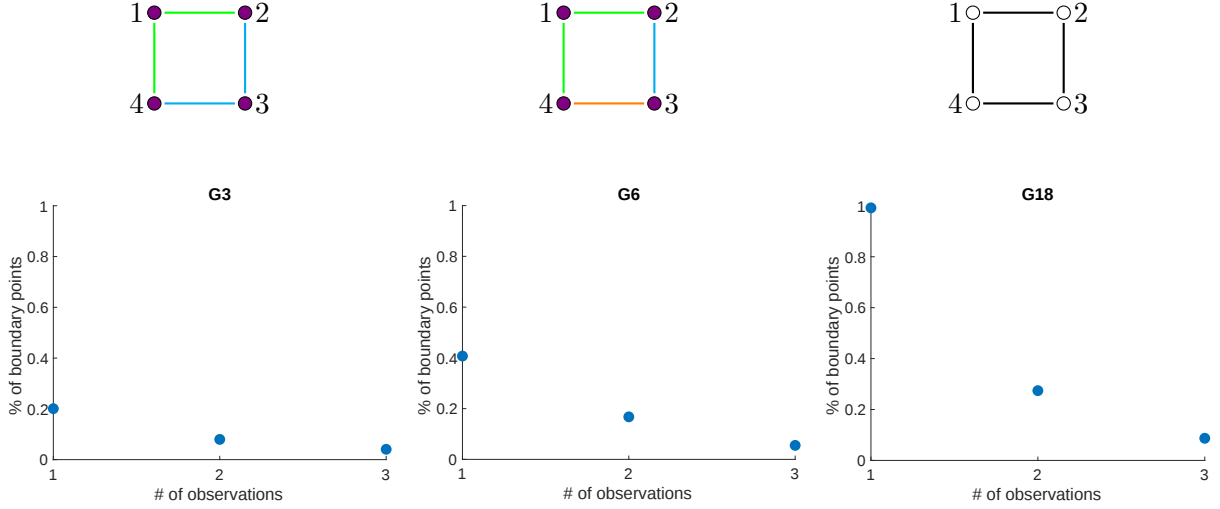
\begin{figure}[htbp]
\centering

% =====================
% Top row: three cycles
% =====================
\begin{minipage}{0.32\textwidth}
\centering
\begin{tikzpicture}[scale=1.2]

    % vertices
    \node (v1) at (0,0) {};
    \node (v2) at (1,0) {};
    \node (v3) at (1,1) {};
    \node (v4) at (0,1) {};

    % labels
    \node[left]  at (v1) {4};
    \node[right] at (v2) {3};
    \node[right] at (v3) {2};
    \node[left]  at (v4) {1};

    % colored edges
    \draw[thick, cyan]   (v1) -- (v2);
    \draw[thick, cyan]    (v2) -- (v3);
    \draw[thick, green]  (v3) -- (v4);
    \draw[thick, green] (v4) -- (v1);

    % colored vertices (original style)
    \node[wViolet] at (v1) {};
    \node[wViolet] at (v2) {};
    \node[wViolet] at (v3) {};
    \node[wViolet] at (v4) {};
\end{tikzpicture}
\end{minipage}%
\hfill
\begin{minipage}{0.32\textwidth}
\centering
\begin{tikzpicture}[scale=1.2]
    % vertices
    \node (v1) at (0,0) {};
    \node (v2) at (1,0) {};
    \node (v3) at (1,1) {};
    \node (v4) at (0,1) {};

    % labels
    \node[left]  at (v1) {4};
    \node[right] at (v2) {3};
    \node[right] at (v3) {2};
    \node[left]  at (v4) {1};

    % colored edges
    \draw[thick, orange] (v1) -- (v2);
    \draw[thick, cyan]  (v2) -- (v3);
    \draw[thick, green]   (v3) -- (v4);
    \draw[thick, green]    (v4) -- (v1);

    % colored vertices
    \node[wViolet] at (v1) {};
    \node[wViolet] at (v2) {};
    \node[wViolet] at (v3) {};
    \node[wViolet] at (v4) {};
\end{tikzpicture}
\end{minipage}%
\hfill
\begin{minipage}{0.32\textwidth}
\centering
\begin{tikzpicture}[scale=1.2]
    % vertices
    \node (v1) at (0,0) {};
    \node (v2) at (1,0) {};
    \node (v3) at (1,1) {};
    \node (v4) at (0,1) {};

    % labels
    \node[left]  at (v1) {4};
    \node[right] at (v2) {3};
    \node[right] at (v3) {2};
    \node[left]  at (v4) {1};

    % colored edges
    \draw[thick]  (v1) -- (v2);
    \draw[thick] (v2) -- (v3);
    \draw[thick]    (v3) -- (v4);
    \draw[thick]   (v4) -- (v1);

    % colored vertices
    \node[wWhite] at (v1) {};
    \node[wWhite] at (v2) {};
    \node[wWhite] at (v3) {};
    \node[wWhite] at (v4) {};
\end{tikzpicture}
\end{minipage}

\vspace{0.8cm}  % space between top and bottom rows

% ======================================
% Bottom row: your existing three figures
% ======================================
\begin{minipage}{0.32\textwidth}
\centering
 \includegraphics[width=\textwidth]{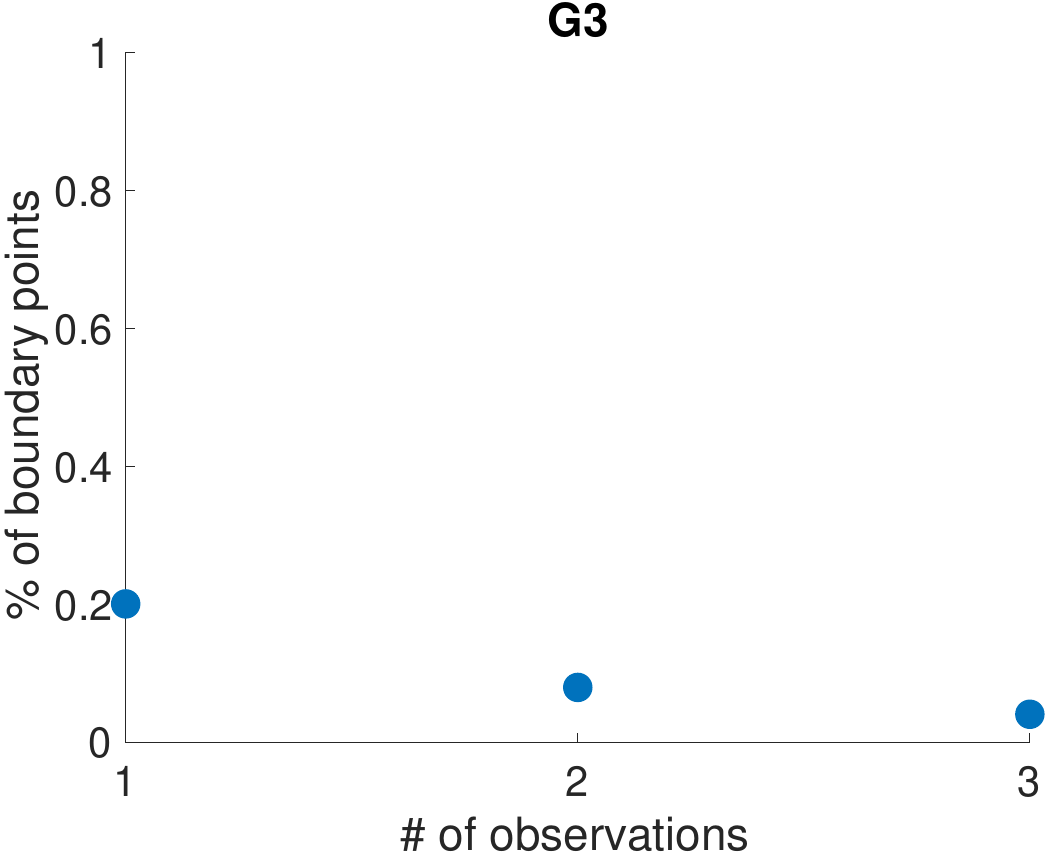}
\end{minipage}%
\hfill
\begin{minipage}{0.32\textwidth}
\centering
 \includegraphics[width=\textwidth]
{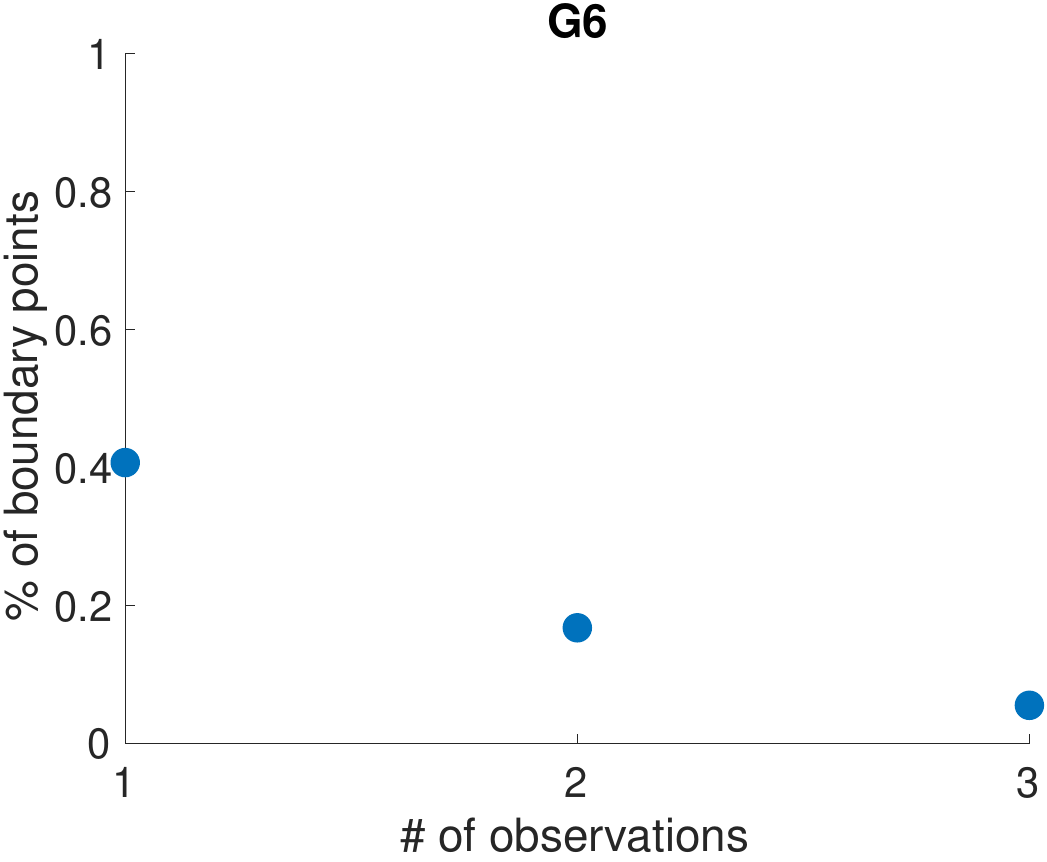}
\end{minipage}%
\hfill
\begin{minipage}{0.32\textwidth}
\centering
 \includegraphics[width=\textwidth]
{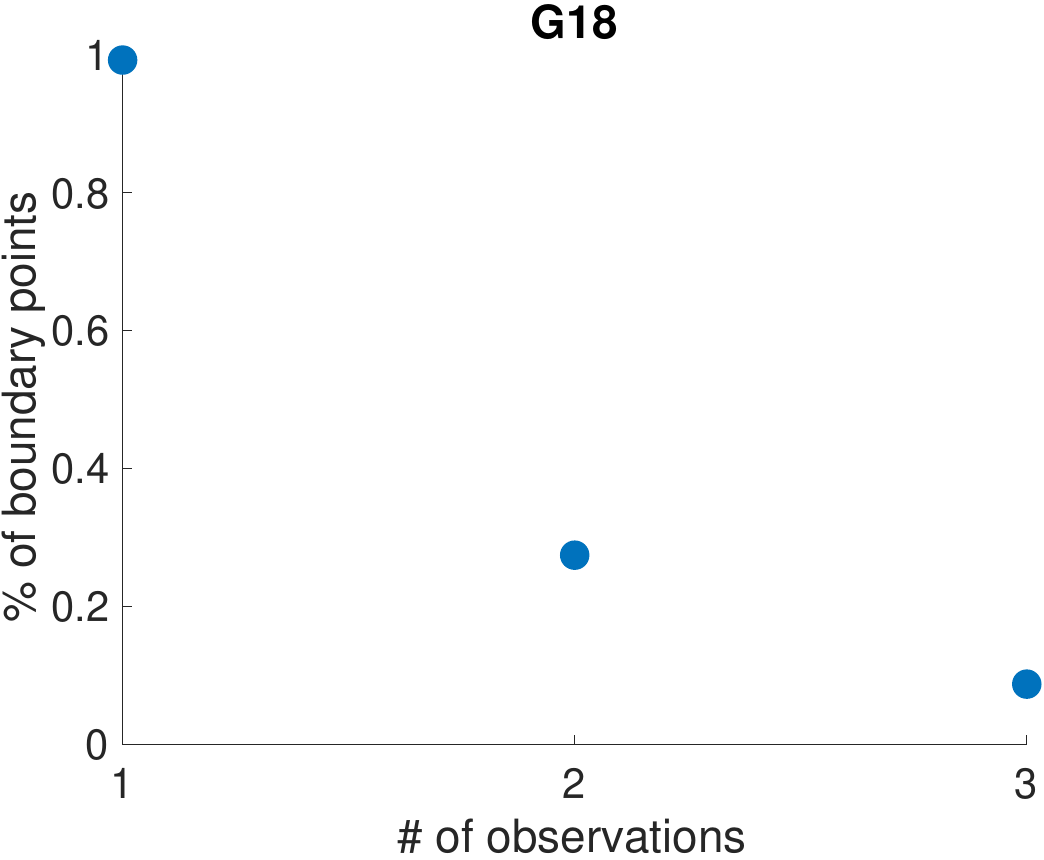}
\end{minipage}

\caption{From left to right, percentage of boundary points in $G_3$, $G_6$ and $G_{18}$ (uncolored 4-cycle). 
We determine the boundary points based on persistent homology computations in dimension 1 for $G_3$, dimensions 1 and 2 for $G_6$, and dim 2 for $G_{18}$.}
\label{fig:combined-cycles}
\end{figure}

\begin{figure}[h]
  \centering
  \includegraphics[width=0.9\textwidth]{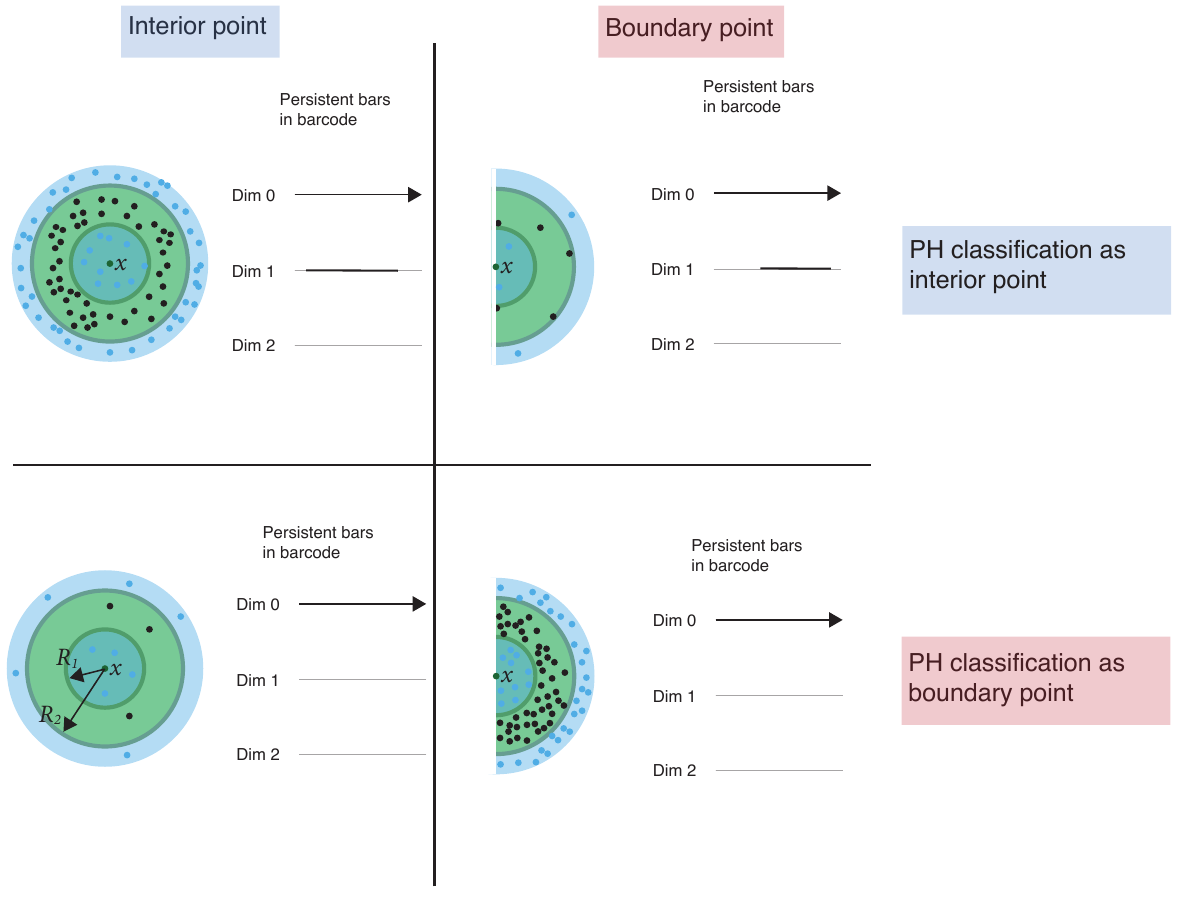}
  \caption{Possible classifications of interior points and boundary points. We show an example of possible classification of points sampled from a surface. Sparsity is a source of misclassifications for both interior and boundary points.}
  \label{fig:InteriorvsBoundary}
\end{figure}

When $A_{XX^T}=\pi_G(\mathcal{S}_{\mathrm{PD}})\cap N(\pi_G(XX^T),R_1,R_2)$ is not sufficiently dense, \Cref{alg:MLT estimation for large graphs} can produce false results. On the one hand, if there are insufficiently many points, \Cref{alg:local homology of sufficient stats} may be unable to build a $k$-sphere around an interior point, which results in misclassification of an interior point as a boundary point. On the other hand, when the sampling is not dense enough to create a contractible disk, it can classify a boundary point as an interior point. We show potential misclassifications due to sparsity in Fig. \ref{fig:InteriorvsBoundary}. 

\subsection{Sampling the cone of sufficient statistics}
\label{subsec:local-uniformization}

Local persistence homology methods are prone to misclassification of boundary and interior points if data is not dense enough in each annular region, as shown in Figure 5. Therefore, it is fundamental to understand how to sample points in the cone of sufficient statistics $\mathcal{C}_G$. 

Given a sample \(X_1,\dots,X_n \in \mathbb{R}^m\) with $n<m$, we call the sufficient statistic of its sample covariance matrix $S$ the \emph{target statistic} $\pi_G(S)$. We investigated global sampling (across $[-1,1]^{n\times m}$, the positive definite cone $\mathrm{PD}(m)$ and Cholesky space) and local sampling around the target statistic via perturbations  of \(X_1,\dots,X_n \in \mathbb{R}^m\), $S$ or Cholesky decompositions of $S$.
Our global experiments reveal a dichotomy depending on the sampling strategy: generating observations uniformly causes sample covariance projections to concentrate along the boundary $\partial\mathcal{C}_G$, while uniform sampling across the positive definite cone $\mathrm{PD}(m)$ does the exact opposite, clustering near the center and leaving the boundary sparse, see Figures~\ref{fig:graph1_suffStats}--\ref{fig:graph6_suffStats}. At the same time, local sampling via perturbations, such as adding $m-n$ observations of small magnitude to the data sample, also suffers from the clustering bias towards the center. 

The main challenge of the naive sampling methods is that even in small graphs, the pre-image of sufficient statistics lies in a high-dimensional space and for a given topologically representative neighborhood of $\pi_G(S)$, its pre-image lies in a small portion of the space. We illustrate this for 4-cycle $G_3$ and parametrization via Cholesky decomposition in Figure~\ref{fig:G3-SuffStats}. Here, the interior of the cone $\mathcal{C}_G$ was obtained by first sampling uniformly from  $\mathrm{PD}(4)$ with trace 1. We choose a $\Sigma_0$ from this sample such that its projection is centrally positioned in $\mathcal{C}_G$. We study the Cholesky decompositions of the sample points whose projections are in the neighborhood of the target statistic and put them on the $\{-\infty,0,\infty\}^{10}$ grid of deviations from the Cholesky decomposition of $\Sigma_0$. The histogram of counts of pre-images of sample points in the grid cells (Figure~\ref{fig:graph3_hist_Cholesky} in \Cref{sec:hist}) shows that the majority of cells have close to 0 observations and there are individual cells in the tail of the distribution with large number of pre-images of sample points. In this case, sampling only from the Cholesky grid cell with the maximum number of pre-images already provides a topologically representative neighborhood, as seen in Figure~\ref{fig:G3_synSuffStats}.

\begin{figure}[H]
  \centering
  \includegraphics[width=0.9\textwidth]{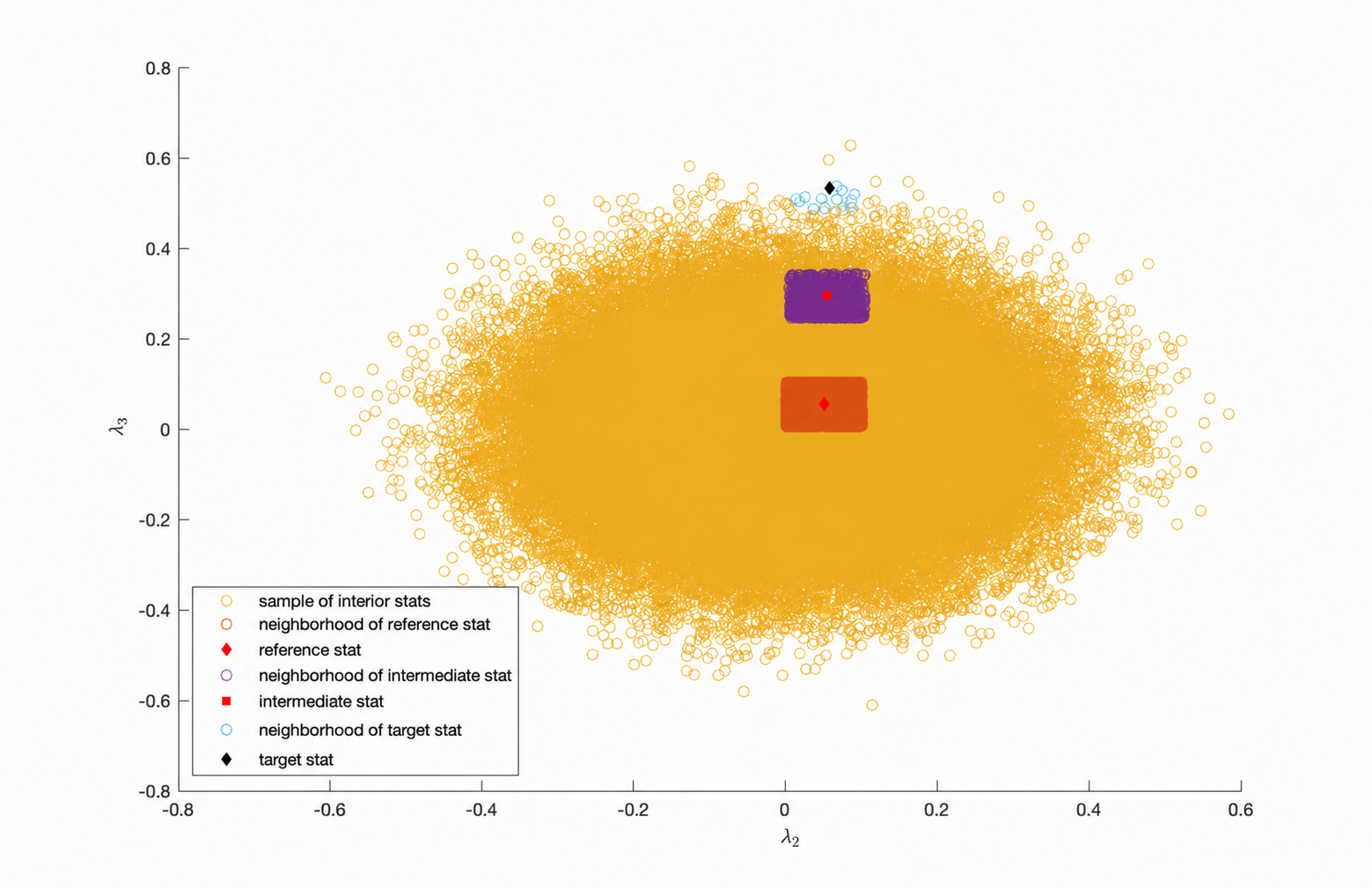}
  \caption{Neighborhoods of the reference, intermediate and target statistics in the cone of sufficient statistics for $G_3$.}
  \label{fig:G3-SuffStats}
\end{figure}

\begin{figure}[H]
  \centering
  \includegraphics[width=0.9\textwidth]{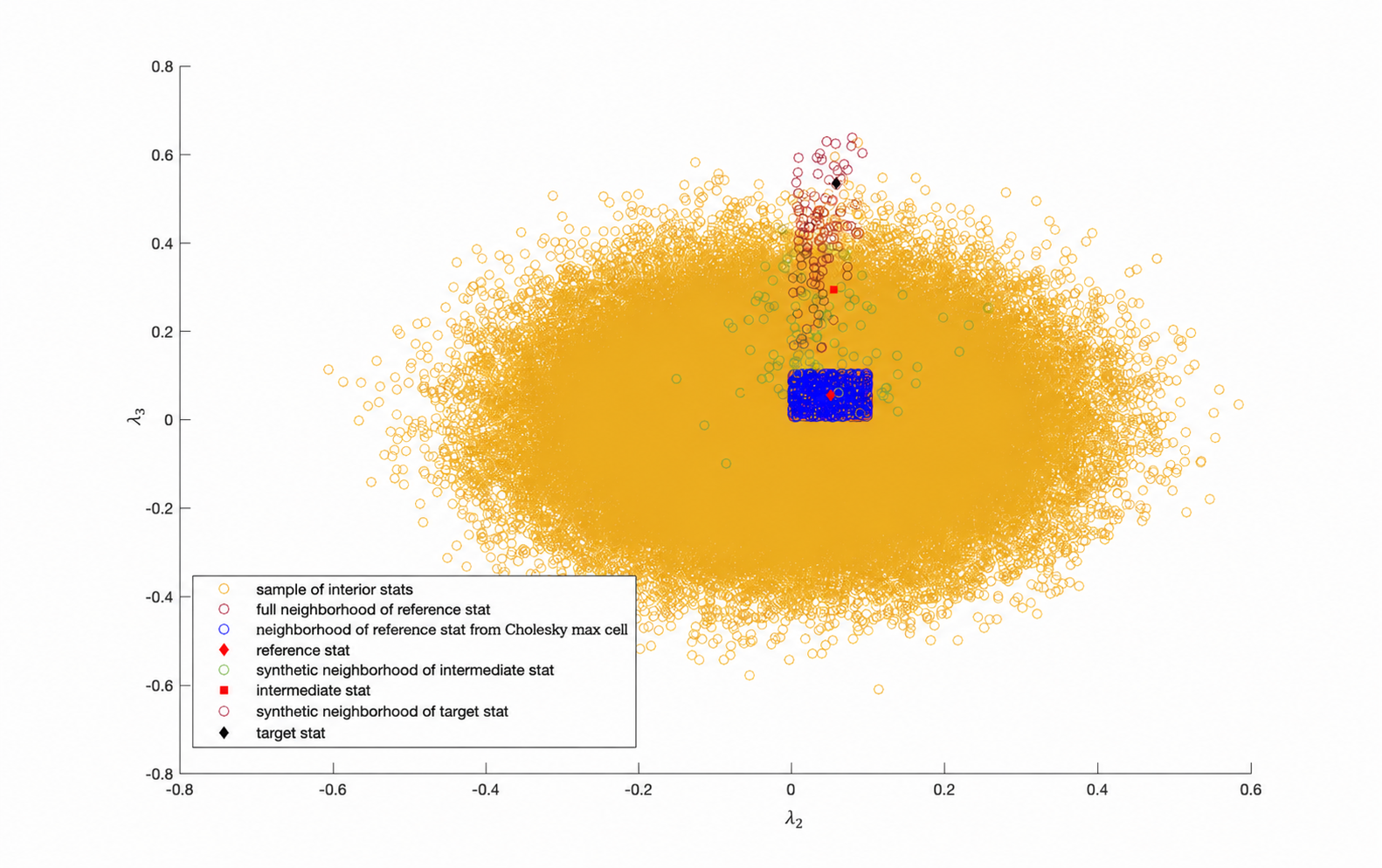}
  \caption{Synthetic neighborhoods of the target and intermediate statistics fixed in Fig.~\ref{fig:G3-SuffStats} for $G_3$. }
  \label{fig:G3_synSuffStats}
\end{figure}

This section introduces a method based on finding a relevant region in the Cholesky space to generate synthetic points that represent the local topology near the boundary of the cone. To ensure the uniqueness of Cholesky decomposition, we approximate $S$ by sufficiently close  $\widetilde \Sigma \in \mathrm{PD}(m)$. Let $\mathcal{S} \subset \mathrm{PD}(m)$ be a uniform sample from $\mathrm{PD}(m)$ with trace $1$ acquired using the algorithm in~\cite{mittelbach2012sampling}. We choose a \emph{reference statistic} $\pi_G(\Sigma_0)$ with $\Sigma_0\in\mathcal{S}$ and $\pi_G(\Sigma_0)$ centrally positioned in $\pi_G(\mathcal{S})$. This ensures the existence of dense sample neighborhood of $\pi_G(\Sigma_0)$ though the local geometry of its Cholesky decomposition
does not reflect the local geometry near the Cholesky decomposition of $\widetilde \Sigma$.

Consider the segment $p(\lambda)=(1-\lambda)\pi_G(\Sigma_0)+\lambda\pi_G(\widetilde \Sigma)$ with $\lambda\in[0,1]$.
Along this path we fix $0<\lambda^\ast<1$ and use the corresponding \emph{intermediate statistic} $\pi_G(\Sigma_{\mathrm{mid}})$ with
$\Sigma_{\mathrm{mid}} \in \mathcal{S}$ such that $d(p(\lambda^\ast), \pi_G(\Sigma_{\mathrm{mid}}))=\min_{\Sigma \in \mathcal{S}} d(p(\lambda^\ast), \pi_G(\Sigma))$. The goal is to identify the intermediate statistic whose projected neighborhood
has enough points in $\pi_G(\mathcal{S})$ to estimate local Cholesky behavior,
and lies sufficiently close to $\pi_G(\widetilde \Sigma)$ that it resembles the local behavior of the Cholesky decomposition of $\widetilde \Sigma$, as illustrated in Figure~\ref{fig:G3-SuffStats}.

We start by identifying the most populated Cholesky grid cells of the intermediate statistic following Algorithm~\ref{alg:intermediate-grid}. Let $\ell_{\mathrm{mid}}$ be the flattened vector of the Cholesky factor of $\Sigma_{\mathrm{mid}}$. 
For each PD matrix $\Sigma_j \in \mathcal{S}$ whose sufficient statistic $\pi_G(\Sigma_j)$ is in the neighborhood of $\pi_G(\Sigma_{\mathrm{mid}})$,
we flatten its Cholesky factor to a vector $\ell_j$ and compute deviations
$d_j=\ell_j-\ell_{\mathrm{mid}}$.
Binning these deviations into the grid with edges $\{-\infty,0,\infty\}^{m(m+1)/2}$
identifies a small set of relevant Cholesky coordinates.

Once the populated Cholesky grid cells have been identified around $\Sigma_{\mathrm{mid}}$, we construct a synthetic neighborhood by sampling Cholesky coordinates uniformly within the corresponding grid cells with Algorithm~\ref{alg:synthetic-neighborhood}. Each synthetic Cholesky factor is converted into a positive-definite matrix $\widehat{\Sigma}$, with positivity ensured by truncating the lower bounds of the diagonal coordinates, and mapped to the corresponding sufficient statistic $\widehat{\sigma}=\pi_G(\widehat{\Sigma})$.

Since the intermediate statistic is selected along the line segment connecting the reference statistic and the positive-definite approximation $\widetilde{\Sigma}$ of the target matrix, the resulting synthetic statistics provide an approximation of the local neighborhood of the target statistic. This synthetic neighborhood captures the local topological structure more faithfully than the neighborhood obtained directly from empirical covariance matrices, see Figure~\ref{fig:G3_synSuffStats}.

To evaluate the algorithms, we sampled 1000 PSD matrices derived from 1, 2, and 3 observations for the graphs in \cite[Table 2]{Uhler}. As shown in \Cref{fig:inhull_orig_synthetic}, synthetic neighborhoods are better at representing the local geometry of sufficient statistics whenever the graph has more than a single vertex color and the performance generally improves in the number of observations. At the same time, \Cref{fig:emprical_prob_all} shows that the performance of Algorithms~\ref{alg:intermediate-grid} and \ref{alg:synthetic-neighborhood} is sensitive to the number of color classes and the size of the data sample. Since the quality of the output of Algorithm~\ref{alg:intermediate-grid} depends largely on the ability to sample a dense representation of the PD cone, an approach to address this limitation in applications is to perform the computationally intense generation of the representation of the PD cone in a centralized environment with high computational power and make it available for distributed environments with limited computational power as input to local computation with Algorithm~\ref{alg:synthetic-neighborhood}.

\begin{algorithm}[H]
    \caption{Construction of Cholesky grid cells around an intermediate statistic}
    \label{alg:intermediate-grid}

    \KwData{
      $G$ -- colored graph on $[m]$\;

      $(X_1,\dots,X_n)$, with $X_i\in\mathbb{R}^m$ and $n<m$ --
      observations whose sufficient statistic $\pi_G(S)$ is the target
      statistic, where $S$ is the sample covariance matrix\;

      $\mathcal{S}\subset\mathrm{PD}(m)$ -- a uniform sample of
      positive-definite matrices with trace $1$\;

      $\pi_G(\Sigma_0)$ -- a reference statistic that is centrally positioned
      in $\pi_G(\mathcal{S})$, with $\Sigma_0\in\mathcal{S}$\;

      $\varepsilon>0$ -- tolerance for the positive-definite approximation of
      the target matrix\;
    }

    \KwResult{
      The intermediate Cholesky vector $\ell_{\mathrm{mid}}$ and a collection
      $\mathcal{C}$ of populated Cholesky grid cells
    }

    \Begin{
      Choose $\widetilde{\Sigma}\in\mathrm{PD}(m)$ such that $\|\widetilde{\Sigma}-S\|<\varepsilon$, where $\widetilde{\Sigma}$ is a PD approximation of the
      target PSD matrix $S$\;

      Construct the line segment
      \[
          p(\lambda)
          =
          (1-\lambda)\pi_G(\Sigma_0)
          +
          \lambda\pi_G(\widetilde{\Sigma}),
          \qquad
          \lambda\in[0,1];
      \]

      Choose $\lambda^\ast\in[0,1]$ such that $p(\lambda^\ast)$ has a
      sufficiently dense neighborhood $\mathcal{N}^\ast$ in
      $\pi_G(\mathcal{S})$\;

      Identify the intermediate matrix
      $\Sigma_{\mathrm{mid}}\in\mathcal{S}$ by
      \[
          \Sigma_{\mathrm{mid}}
          \in
          \arg\min_{\Sigma\in\mathcal{S}}
          d\!\left(
              p(\lambda^\ast),
              \pi_G(\Sigma)
          \right);
      \]

      Extract all matrices $\Sigma_j\in\mathcal{S}$ such that
      \[
          \pi_G(\Sigma_j)\in\mathcal{N}^\ast;
      \]

      Compute the upper-triangular Cholesky factor
      $U_{\mathrm{mid}}$ of $\Sigma_{\mathrm{mid}}$ and flatten its
      upper-triangular entries into
      \[
          \ell_{\mathrm{mid}}\in\mathbb{R}^{q},
          \qquad
          q=\frac{m(m+1)}{2};
      \]

      \ForEach{extracted matrix $\Sigma_j$}{
          Compute its upper-triangular Cholesky factor $U_j$\;

          Flatten the upper-triangular entries of $U_j$ into
          $\ell_j\in\mathbb{R}^{q}$\;

          Compute the deviation
          \[
              d_j=\ell_j-\ell_{\mathrm{mid}};
          \]

          Assign $d_j$ to a grid cell
          $g_j=(g_{j,1},\dots,g_{j,q})\in\{-,+\}^{q}$, where
          \[
              g_{j,k}
              =
              \begin{cases}
                  -, & d_{j,k}<0,\\
                  +, & d_{j,k}\geq 0;
              \end{cases}
          \]
      }

      Identify a collection $\mathcal{C}$ of sufficiently populated
      Cholesky grid cells\;

      \KwRet{$\ell_{\mathrm{mid}},\mathcal{C}$}
    }
\end{algorithm}

\begin{algorithm}[H]
\caption{Synthetic neighborhood generation from Cholesky grid cells}
\label{alg:synthetic-neighborhood}

\KwData{
Flattened Cholesky factor $\ell_{\mathrm{mid}}\in\mathbb{R}^{q}$ of the
intermediate matrix, where $q=m(m+1)/2$\;

Selected Cholesky grid cells $\mathcal C\subseteq\{-,+\}^{q}$\;

Perturbation scale $\delta$\;

Number of synthetic samples per grid cell $N_{\mathrm{cell}}$\;
}

\KwResult{
Synthetic neighborhood $\widehat{\mathcal N}$ approximating the local
topology of the target statistic
}

\Begin{

Initialize $\widehat{\mathcal N}\leftarrow\varnothing$\;

\ForEach{$g\in\mathcal C$}{
    \For{$r=1$ \KwTo $N_{\mathrm{cell}}$}{

        \For{$k=1$ \KwTo $q$}{
            Sample $\widehat\ell_k^{(g,r)}$ uniformly from the
            half-interval determined by $g_k$; if $k$ corresponds to a
            diagonal Cholesky entry, truncate the lower endpoint to
            $10^{-6}$\;
        }

        Form the upper-triangular Cholesky factor
        $\widehat U^{(g,r)}$ from
        $\widehat\ell^{(g,r)}$\;

        Compute
        $\widehat\Sigma^{(g,r)}
        =
        (\widehat U^{(g,r)})^{\mathsf T}
        \widehat U^{(g,r)}$
        and rescale it to have trace $1$\;

        Compute the synthetic statistic
        $\widehat \sigma^{(g,r)}
        =
        \pi_G(\widehat\Sigma^{(g,r)})$\;

        Add $\widehat \sigma^{(g,r)}$ to
        $\widehat{\mathcal N}$\;
    }
}

\KwRet{$\widehat{\mathcal N}$}

}
\end{algorithm}

\begin{figure}[H]
  \centering
  \includegraphics[width=0.7\textwidth]{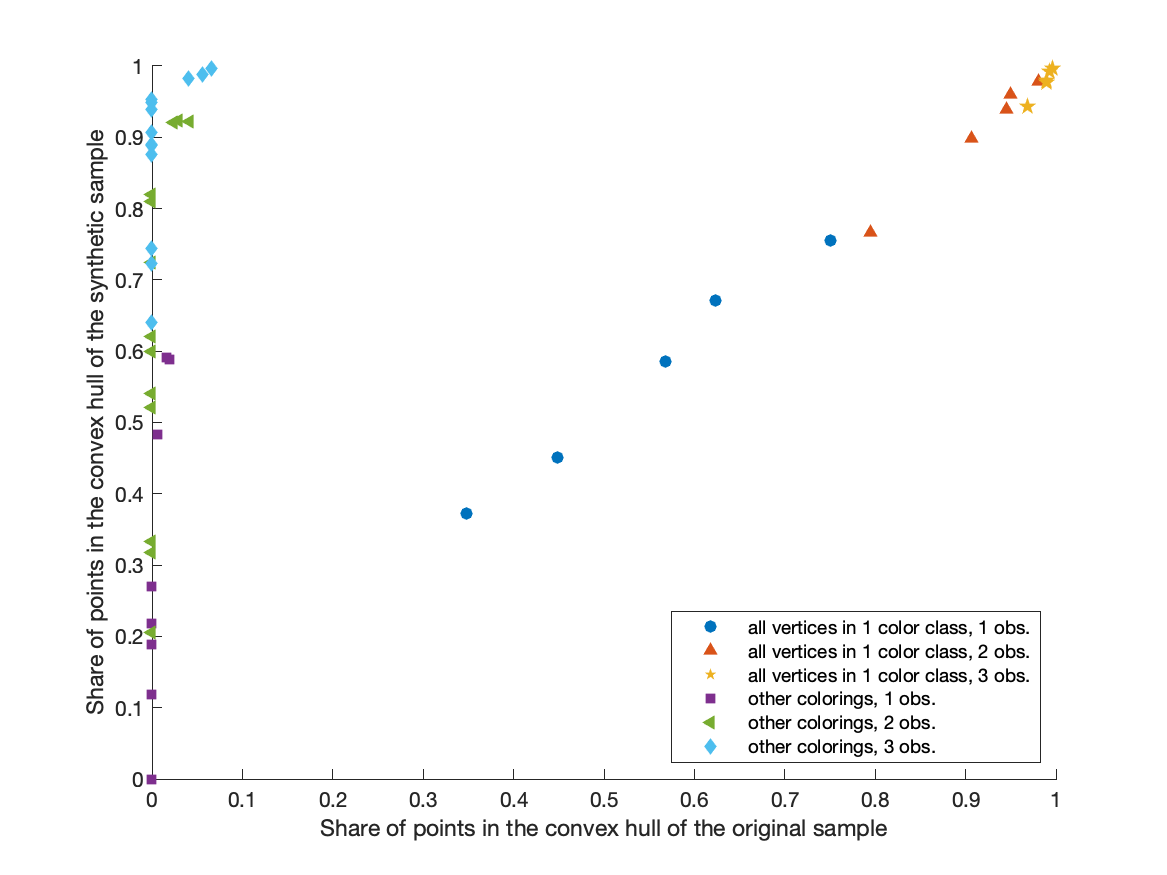}
  \caption{Comparison of original and synthetic neighbourhoods of sufficient statistics of PSD matrices for the graphs in \cite[Table 2]{Uhler}.}
  \label{fig:inhull_orig_synthetic}
\end{figure}

\begin{figure}[H]
  \centering
  \includegraphics[width=0.7\textwidth]{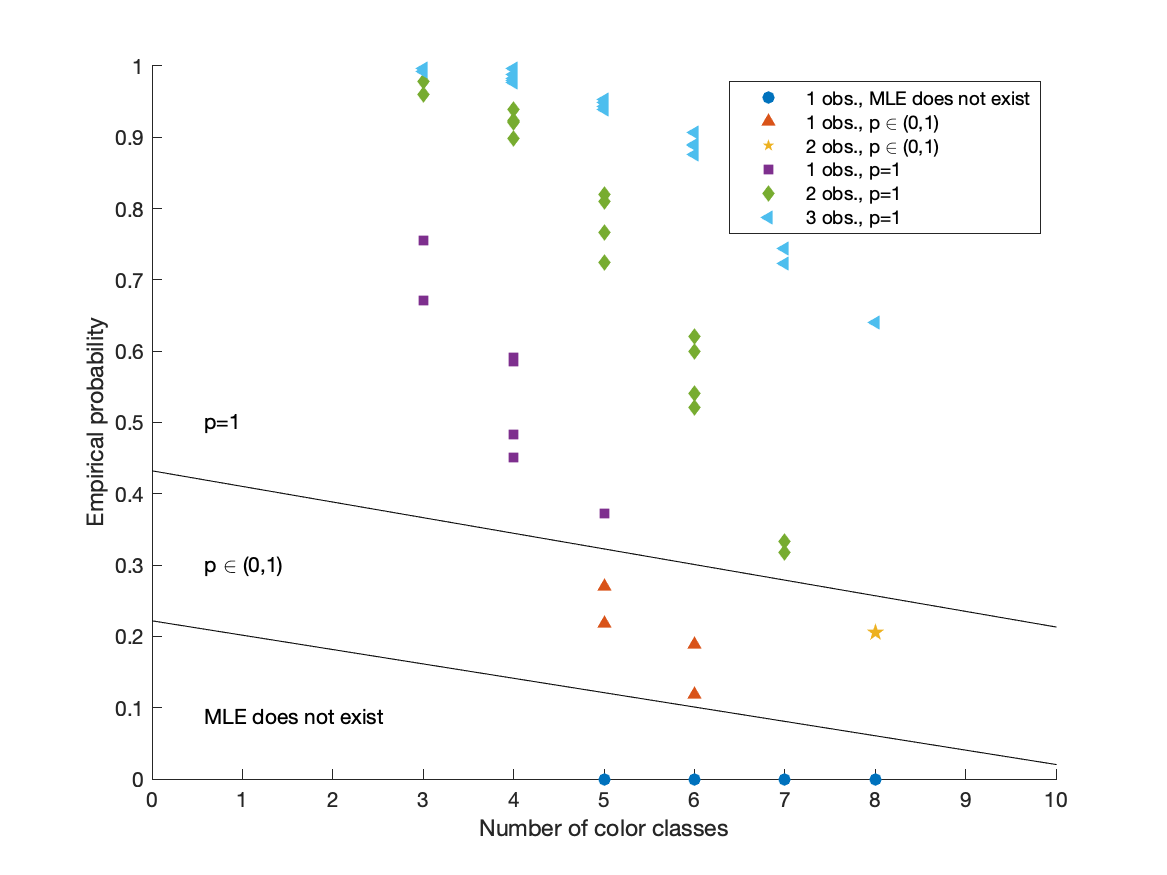}
  \caption{ Empirical probability $p$ of the existence of MLE for the 4-cycles in \cite[Table 2]{Uhler}, ranging from 3 to 8 color classes. The sample used to compute $p$ has been obtained with \Cref{alg:synthetic-neighborhood}. MLE existence is determined using convex hull containment of the target statistic in the synthetic sample.}
  \label{fig:emprical_prob_all}
\end{figure}

\section{Discussion}

Efficient computation of the maximum likelihood threshold ($\mathrm{mlt}(G)$) for a given CGGM currently remains an open challenge. 
However, the generic completion rank ($\mathrm{gcr}(G)$) can be computed in random polynomial time with an adapted version of \cite[Algorithm~3.2]{Bernstein2020GCR}. While Bernstein et al.\ \cite{Bernstein2022MLT} conjecture that $\mathrm{gcr}(G)=\mathrm{mlt}(G)$ for almost all uncolored $G$, we are still far from understanding the broader landscape in the colored setting. 

A first analysis shows that for colored graphs with three vertices, the equality $\mathrm{gcr}(G) = \mathrm{mlt}(G)$ always holds, and $\mathrm{mlt}(G) = \mathrm{wmlt}(G)$ in 19 out of the 25 graphs. This raises fundamental questions:
\begin{itemize}
\item How often do we have $\mathrm{gcr}(G)=\mathrm{mlt}(G)$?
\item Can these relationships be characterized in terms of the coloring of the graph?
\end{itemize}
\noindent
Regardless of these considerations, there exist sufficient instances with $\mathrm{mlt}(G)<\mathrm{gcr}(G)$ to justify pursuing methods for computing the MLT directly.
Here, we explored computational ideas from algebraic geometry and TDA to address these questions.

From the perspective of algebraic geometry, we focused on two approaches:
\begin{enumerate}
    \item \emph{Computing the vanishing ideal of the projection map.} This approach is conceptually powerful: obtaining the vanishing ideal provides complete algebraic information about the model. However, its scalability is limited. The main bottlenecks are  
    (1) Gröbner basis computations required to determine the ideal, and  
    (2) the implementation of positive certificates via sums of squares (SOS).
    \begin{itemize}
    \item \textbf{Direction for future work 1.} Investigate numerical methods for computing the vanishing ideal using witness sets.
    \item \textbf{Direction for future work 2.} Explore alternative positivity certificates beyond SOS and evaluate their computational feasibility.
    \end{itemize}

    \item \emph{Solving matrix completion problems.} This method can be pursued either by directly finding positive definite completions 
    (\Cref{th:Sinn1}), or by using the equivalent characterization based 
    on kernels
    (\Cref{th:Sinn2}). This perspective provides powerful tools for computing $\mathrm{mlt}(G)$ in specific graph families and is often more tractable when computing $\mathrm{wmlt}(G)$. Nonetheless, general-purpose algorithms for this approach have yet to be developed.
    \begin{itemize}
    \item \textbf{Direction for future work 3.} Exploit \Cref{th:Sinn1} and \Cref{th:Sinn2} to design algorithms for computing $\mathrm{mlt}(G)$ and $\mathrm{wmlt}(G)$.
    \end{itemize}
\end{enumerate}

From the TDA perspective, we investigated whether persistent homology of annular neighbourhoods can be used to detect points close to the boundary of the cone of sufficient statistics. We find that this approach is promising, especially as a way to mitigate the computational challenges faced by algebraic-geometry methods, but the results depend strongly on the sampling method, which must be sufficiently uniform. We propose an algorithm for local uniformization of the near-boundary sufficient statistics. 

However, the method only captures relative trends but does 
not provide sufficiently sharp decision boundaries. 

\begin{itemize}
    \item \textbf{Direction for future work 4.} 
   Investigate alternative methods for neighborhood  construction that allow for variable density in samples, e.g. dynamic choice of radius, $k$-NN algorithms, methods from manifold learning theory \cite{manifold}.
    \item \textbf{Direction for future work 5.} Investigate the theoretical properties of the distribution of active Cholesky coordinates to enable more efficient sampling.
     \item \textbf{Direction for future work 6.} 
     Generalize algorithm in \cite{mittelbach2012sampling} to uniformly sample from $\mathrm{PSD}(m)$ with trace constraint
     and understand theoretical properties of projections onto the cone of sufficient statistics $\mathcal{C}_G$ of uniform distributions in trace-constrained $\mathrm{PD}(m)$ and $\mathrm{PSD}(m)$.
    \end{itemize}

\section{Code and data availability}
All TDA analyses were implemented in {\sc Matlab} using ripser \cite{bauer2021ripser} for persistent homology computation. All algebraic geometry computations and algorithms have been performed and implemented in \texttt{Macaulay2} \cite{M2}, and are build on the packages \texttt{GraphicalModels} \cite{garcia2013graphical}, \texttt{GraphicalModelsMLE} \cite{M2GraphicalModelsMLE} and \texttt{SumsOfSquares} \cite{M2SOS}. Source code and experimental data are made available on \url{https://www.dropbox.com/scl/fo/5n4kiff0xv63uw49n3jni/ADR5R9AoOGPUVMazqB9qpb4?rlkey=074sr1m93kdjk6p63iq8uotng&st=an67mvgo&dl=0}. 

\section{Acknowledgments}
The authors express their gratitude to the organizers of Biomat 2023: Multiscale Methods at the Frontier Between Data and Mathematical Models, held at the Centre de Recerca Matemàtica (CRM). In particular, the course on TDA delivered by the third author led to expanding the initial work of the first two authors into this joint collaboration.

R.H. was supported by the postdoctoral fellowship programme Beatriu de Pin\'os (ref. 2021BP00119), funded by the Secretary of
Universities and Research (Government of Catalonia) and partially supported by projects with references PID2019-103849G-I00 and PID2023-146936NB-I00, financed by 
the Spanish State Agency,
MCIN/AEI/10.13039/501100011033/
FEDER,UE, and by the AGAUR project 2021 SGR 00603 Geometry of Manifolds and Applications, GEOMVAP. 

%%
%% Bibliography
%%

%% Please use bibtex, 

\bibliography{JOCG_bibliography}

\newpage

\maketitle

%%
%% Bibliography
%%

%% Please use bibtex, 

%\bibliography{socg-lipics-v2021-paper}

\appendix
\section{Supplementary Material}\label{app:SM}

\subsection{Colored 3 and 4-cycles}

\begin{table}[htbp]
\centering
\caption{Colored 4-cycles from \cite[Table 2]{Uhler} studied in this paper} 
\label{table: all graphs}
\renewcommand{\arraystretch}{1.7}
\resizebox{\textwidth}{!}{%
\begin{tabular}{c c p{3.2cm} c c}
\hline
\textbf{Graph} & $\mathbf{K}$ & \textbf{1 obs.} & \textbf{2 obs.} & $\mathbf{\ge 3\text{ obs.}}$ \\
\hline

% Row 1: G3
\begin{tikzpicture}[-, every node/.style={circle, draw, inner sep=1.5pt, font=\scriptsize}, line width=0.8pt, baseline={([yshift=-0.5ex]current bounding box.center)}]
  \node[draw=none] at (-0.5, 0.8) {$G_3$:};  
  \node[fill=orange] (1) at (0,0) {$4$};
  \node[fill=orange] (2) at (0,0.8) {$1$};
  \node[fill=orange] (3) at (0.8,0.8) {$2$};
  \node[fill=orange] (4) at (0.8,0) {$3$}; 
  \draw[color=cyan, line width=1.5pt] (1) -- (2) -- (3);   
  \draw[color=green, line width=1.5pt] (3) -- (4) -- (1);   
\end{tikzpicture} 
& 
$\left(\begin{smallmatrix}
\lambda_1 & \lambda_2 & 0 & \lambda_2 \\
\lambda_2 & \lambda_1 & \lambda_3 & 0 \\
0 & \lambda_3 & \lambda_1 & \lambda_3 \\
\lambda_2 & 0 & \lambda_3 & \lambda_1
\end{smallmatrix}\right)$ 
& $p=1$ & $p=1$ & $p=1$ \\ \hline

% Row 2: G6
\begin{tikzpicture}[-, every node/.style={circle, draw, inner sep=1.5pt, font=\scriptsize}, line width=0.8pt, baseline={([yshift=-0.5ex]current bounding box.center)}]
  \node[draw=none] at (-0.5, 0.8) {$G_6$:};  
  \node[fill=orange] (1) at (0,0) {$4$};
  \node[fill=orange] (2) at (0,0.8) {$1$};
  \node[fill=orange] (3) at (0.8,0.8) {$2$};
  \node[fill=orange] (4) at (0.8,0) {$3$}; 
  \draw[color=cyan, line width=1.5pt] (1) -- (2) -- (3);   
  \draw (3) -- (4) -- (1);   
\end{tikzpicture} 
& 
$\left(\begin{smallmatrix}
\lambda_1 & \lambda_2 & 0 & \lambda_2 \\
\lambda_2 & \lambda_1 & \lambda_3 & 0 \\
0 & \lambda_3 & \lambda_1 & \lambda_4 \\
\lambda_2 & 0 & \lambda_4 & \lambda_1
\end{smallmatrix}\right)$ 
& $p=1$ & $p=1$ & $p=1$ \\ \hline

% Row 3: G9
\begin{tikzpicture}[-, every node/.style={circle, draw, inner sep=1.5pt, font=\scriptsize}, line width=0.8pt, baseline={([yshift=-0.5ex]current bounding box.center)}]
  \node[draw=none] at (-0.5, 0.8) {$G_9$:};  
  \node[fill=orange] (1) at (0,0) {$4$};
  \node (2) at (0,0.8) {$1$};
  \node[fill=orange] (3) at (0.8,0.8) {$2$};
  \node (4) at (0.8,0) {$3$}; 
  \draw[color=cyan, line width=1.5pt] (1) -- (2) -- (3);   
  \draw (3) -- (4) -- (1);   
\end{tikzpicture} 
& 
$\left(\begin{smallmatrix}
\lambda_1 & \lambda_4 & 0 & \lambda_4 \\
\lambda_4 & \lambda_2 & \lambda_5 & 0 \\
0 & \lambda_5 & \lambda_3 & \lambda_6 \\
\lambda_4 & 0 & \lambda_6 & \lambda_2
\end{smallmatrix}\right)$ 
& No (Prop.~\ref{proposition: resolved conjectures}) & $p=1$ & $p=1$ \\ \hline

% Row 4: G11
\begin{tikzpicture}[-, every node/.style={circle, draw, inner sep=1.5pt, font=\scriptsize}, line width=0.8pt, baseline={([yshift=-0.5ex]current bounding box.center)}]
  \node[draw=none] at (-0.5, 0.8) {$G_{11}$:};  
  \node[fill=orange] (1) at (0,0) {$4$};
  \node[fill=violet] (2) at (0,0.8) {$1$};
  \node[fill=orange] (3) at (0.8,0.8) {$2$};
  \node[fill=violet] (4) at (0.8,0) {$3$}; 
  \draw[color=cyan, line width=1.5pt] (1) -- (2);
  \draw[color=cyan, line width=1.5pt] (3) -- (4);   
  \draw (2) -- (3);
  \draw (4) -- (1);   
\end{tikzpicture} 
& 
$\left(\begin{smallmatrix}
\lambda_1 & \lambda_3 & 0 & \lambda_4 \\
\lambda_3 & \lambda_2 & \lambda_4 & 0 \\
0 & \lambda_4 & \lambda_1 & \lambda_5 \\
\lambda_4 & 0 & \lambda_5 & \lambda_2
\end{smallmatrix}\right)$ 
& No (Prop.~\ref{proposition: resolved conjectures}) & $p=1$ & $p=1$ \\ \hline

% Row 5: G12
\begin{tikzpicture}[-, every node/.style={circle, draw, inner sep=1.5pt, font=\scriptsize}, line width=0.8pt, baseline={([yshift=-0.5ex]current bounding box.center)}]
  \node[draw=none] at (-0.5, 0.8) {$G_{12}$:};  
  \node[fill=orange] (1) at (0,0) {$4$};
  \node[fill=orange] (2) at (0,0.8) {$1$};
  \node[fill=orange] (3) at (0.8,0.8) {$2$};
  \node[fill=orange] (4) at (0.8,0) {$3$}; 
  \draw (1) -- (2) -- (3) -- (4) -- (1);   
\end{tikzpicture} 
& 
$\left(\begin{smallmatrix}
\lambda_1 & \lambda_2 & 0 & \lambda_5 \\
\lambda_2 & \lambda_1 & \lambda_3 & 0 \\
0 & \lambda_3 & \lambda_1 & \lambda_4 \\
\lambda_5 & 0 & \lambda_4 & \lambda_1
\end{smallmatrix}\right)$ 
& $p=1$ (Prop.~\ref{prop:mCycles}) & $p=1$ & $p=1$ \\ \hline

% Row 6: G14
\begin{tikzpicture}[-, every node/.style={circle, draw, inner sep=1.5pt, font=\scriptsize}, line width=0.8pt, baseline={([yshift=-0.5ex]current bounding box.center)}]
  \node[draw=none] at (-0.5, 0.8) {$G_{14}$:};  
  \node[fill=orange] (1) at (0,0) {$4$};
  \node[fill=violet] (2) at (0,0.8) {$1$};
  \node[fill=orange] (3) at (0.8,0.8) {$2$};
  \node[fill=violet] (4) at (0.8,0) {$3$}; 
  \draw (1) -- (2) -- (3) -- (4) -- (1);   
\end{tikzpicture} 
& 
$\left(\begin{smallmatrix}
\lambda_1 & \lambda_3 & 0 & \lambda_6 \\
\lambda_3 & \lambda_2 & \lambda_4 & 0 \\
0 & \lambda_4 & \lambda_1 & \lambda_5 \\
\lambda_6 & 0 & \lambda_5 & \lambda_2
\end{smallmatrix}\right)$ 
& No (Prop.~\ref{proposition: resolved conjectures}) & $p=1$ & $p=1$ \\ \hline

% Row 7: G17
\begin{tikzpicture}[-, every node/.style={circle, draw, inner sep=1.5pt, font=\scriptsize}, line width=0.8pt, baseline={([yshift=-0.5ex]current bounding box.center)}]
  \node[draw=none] at (-0.5, 0.8) {$G_{17}$:};  
  \node[fill=orange] (1) at (0,0) {$4$};
  \node (2) at (0,0.8) {$1$};
  \node[fill=orange] (3) at (0.8,0.8) {$2$};
  \node (4) at (0.8,0) {$3$}; 
  \draw (1) -- (2) -- (3) -- (4) -- (1);   
\end{tikzpicture} 
& 
$\left(\begin{smallmatrix}
\lambda_1 & \lambda_4 & 0 & \lambda_7 \\
\lambda_4 & \lambda_2 & \lambda_5 & 0 \\
0 & \lambda_5 & \lambda_3 & \lambda_6 \\
\lambda_7 & 0 & \lambda_6 & \lambda_2
\end{smallmatrix}\right)$ 
& No (Prop.~\ref{proposition: resolved conjectures}) & $p=1$ & $p=1$ \\ \hline

% Row 8: G18
\begin{tikzpicture}[-, every node/.style={circle, draw, inner sep=1.5pt, font=\scriptsize}, line width=0.8pt, baseline={([yshift=-0.5ex]current bounding box.center)}]
  \node[draw=none] at (-0.5, 0.8) {$G_{18}$:};  
  \node (1) at (0,0) {$4$};
  \node (2) at (0,0.8) {$1$};
  \node (3) at (0.8,0.8) {$2$};
  \node (4) at (0.8,0) {$3$}; 
  \draw (1) -- (2) -- (3) -- (4) -- (1);   
\end{tikzpicture} 
& 
$\left(\begin{smallmatrix}
\lambda_1 & \lambda_5 & 0 & \lambda_8 \\
\lambda_5 & \lambda_2 & \lambda_6 & 0 \\
0 & \lambda_6 & \lambda_3 & \lambda_7 \\
\lambda_8 & 0 & \lambda_7 & \lambda_4
\end{smallmatrix}\right)$ 
& No & $p\in(0,1)$ & $p=1$ \\ \hline

\end{tabular}%
}
\end{table}

\begin{landscape}
\begin{table}[p]
\centering
\caption{ML thresholds for all colored 3-cycles, together with their elimination ideals.}
\label{table:3cycles}
\renewcommand{\arraystretch}{1.2}
\setlength{\tabcolsep}{4pt} 

\resizebox{0.95\linewidth}{!}{%
\begin{tabular}{|p{3.2cm}|c|c|c|c|p{8.5cm}|c|p{4.8cm}|}
\hline
\textbf{Graph} & \textbf{Conc. matrix} & \textbf{dim} & \textbf{WMLT} & \textbf{MLT} & \textbf{$I_{G,1}$} & \textbf{$I_{G,2}$} & \textbf{Sufficient statistics} \\
\hline

uncolored & 
$\left(\begin{smallmatrix} \lambda_1 & \lambda_4 & \lambda_5\\ \lambda_4 & \lambda_2 &\lambda_6\\ \lambda_5 & \lambda_6 &\lambda_3 \end{smallmatrix}\right)$ 
& 6 & 3 & 3 & 
\small $t_4^2-4t_2t_1, t_4^2-4t_2t_1, t_5t_4-2t_6t_1, t_6t_4-2t_5t_2, t_5^2-4t_3t_1, t_6t_5-2t_4t_3, t_6^2-4t_3t_2$ 
& \parbox{2.8cm}{\raggedright \scriptsize $t_6t_5t_4-t_4^2t_3-t_5^2t_2-t_6^2t_1+4t_3t_2t_1$}
& \small $t_1=s_{11}, t_2=s_{22}, t_3=s_{33}$,\newline
$t_4=2s_{12}, t_5=2s_{13}, t_6=2s_{23}$ \\ \hline

2 vertices & 
$\left(\begin{smallmatrix} \lambda_1 & \lambda_3 & \lambda_4\\ \lambda_3 & \lambda_1 &\lambda_5\\ \lambda_4 & \lambda_5 &\lambda_2 \end{smallmatrix}\right)$ 
& 5 & 2 & 2 & 
\small $t_4t_3-2t_5t_2$, $t_4^2+t_3^2-4t_5t_1$ 
& 0 
& \small $t_1=s_{11}+s_{22}, t_2=s_{33}$,\newline
$t_3=2s_{12}, t_4=2s_{13}, t_5=2s_{23}$ \\ \hline

2 edges & 
$\left(\begin{smallmatrix} \lambda_1 & \lambda_4 & \lambda_4\\ \lambda_4 & \lambda_2 &\lambda_5\\ \lambda_4 & \lambda_5 &\lambda_3 \end{smallmatrix}\right)$ 
& 5 & 2 & 2 & 
\small $t_4^2-4t_5t_1-4t_3t_1-4t_2t_1$, $t_5^2-4t_3t_2$ 
& 0 
& \small $t_1=s_{11}, t_2=s_{22}, t_3=s_{33}$,\newline
$t_4=2s_{12}+2s_{13}, t_5=2s_{23}$ \\ \hline

2 vert. 2 edges\newline(excl. edge) & 
$\left(\begin{smallmatrix} \lambda_1 & \lambda_3 & \lambda_4\\ \lambda_3 & \lambda_1 &\lambda_4\\ \lambda_4 & \lambda_4 &\lambda_2 \end{smallmatrix}\right)$ 
& 4 & 2 & 2 & 
\small $t_4^2-4t_3t_2-4t_2t_1$ 
& 0 
& \small $t_1=s_{11}+s_{22}, t_2=s_{33}$,\newline
$t_3=2s_{12}, t_4=2s_{13}+2s_{23}$ \\ \hline

2 vert. 2 edges\newline(incl. edge) & 
$\left(\begin{smallmatrix} \lambda_1 & \lambda_3 & \lambda_3\\ \lambda_3 & \lambda_1 &\lambda_4\\ \lambda_3 & \lambda_4 &\lambda_2 \end{smallmatrix}\right)$ 
& 4 & 1 & 2 & 
\small $t_4^4+4t_4^3t_2+4t_4^2t_2^2+4t_3^2t_2^2-4t_4^2t_2t_1-16t_4t_2^2t_1-16t_2^3t_1$ 
& 0 
& \small $t_1=s_{11}+s_{22}, t_2=s_{33}$,\newline
$t_3=2s_{12}+2s_{13}, t_4=2s_{23}$ \\ \hline

3 vertices & 
$\left(\begin{smallmatrix} \lambda_1 & \lambda_2 & \lambda_3\\ \lambda_2 & \lambda_1 &\lambda_4\\ \lambda_3 & \lambda_4 &\lambda_1 \end{smallmatrix}\right)$ 
& 4 & 1 & 2 & 
\small $t_4^2t_3^2+t_4^2t_2^2+t_3^2t_2^2-2t_4t_3t_2t_1$ 
& 0 
& \small $t_1=s_{11}+s_{22}+s_{33}$,\newline
$t_4=2s_{12}, t_5=2s_{13}, t_6=2s_{23}$ \\ \hline

3 edges & 
$\left(\begin{smallmatrix} \lambda_1 & \lambda_4 & \lambda_4\\ \lambda_4 & \lambda_2 &\lambda_4\\ \lambda_4 & \lambda_4 &\lambda_3 \end{smallmatrix}\right)$ 
& 4 & 1 & 2 & 
\small $t_4^4-8t_4^2t_3t_2+16t_3^2t_2^2-8t_4^2t_3t_1-8t_4^2t_2t_1-64t_4t_3t_2t_1-32t_3^2t_2t_1-32t_3t_2^2t_1+16t_3^2t_1^2-32t_3t_2t_1^2+16t_2^2t_1^2$ 
& 0 
& \small $t_1=s_{11}, t_2=s_{22}, t_3=s_{33}$,\newline
$t_4=2s_{12}+2s_{13}+2s_{23}$ \\ \hline

3 vert. 2 edges & 
$\left(\begin{smallmatrix} \lambda_1 & \lambda_2 & \lambda_2\\ \lambda_2 & \lambda_1 &\lambda_3\\ \lambda_2 & \lambda_3 &\lambda_1 \end{smallmatrix}\right)$ 
& 2 & 1 & 1 & 
\small 0 
& 0 
& \small $t_1=s_{11}+s_{22}+s_{33}$,\newline
$t_2=2s_{12}+2s_{13}, t_3=2s_{23}$ \\ \hline

2 vert. 3 edges & 
$\left(\begin{smallmatrix} \lambda_1 & \lambda_3 & \lambda_3\\ \lambda_3 & \lambda_1 &\lambda_3\\ \lambda_3 & \lambda_3 &\lambda_2 \end{smallmatrix}\right)$ 
& 2 & 1 & 1 & 
\small 0 
& 0 
& \small $t_1=s_{11}+s_{22}, t_2=s_{33}$,\newline
$t_3=2s_{12}+2s_{13}+2s_{23}$ \\ \hline

3 vert. 3 edges & 
$\left(\begin{smallmatrix} \lambda_1 & \lambda_2 & \lambda_2\\ \lambda_2 & \lambda_1 &\lambda_2\\ \lambda_2 & \lambda_2 &\lambda_1 \end{smallmatrix}\right)$ 
& 2 & 1 & 1 & 
\small 0 
& 0 
& \small $t_1=s_{11}+s_{22}+s_{33}$,\newline
$t_2=2s_{12}+2s_{13}+2s_{23}$ \\ \hline

\end{tabular}%
}
\end{table}
\end{landscape}

\subsection{Sampling the cone of sufficient statistics}\label{sec:graphs}

\begin{figure}[h]
  \centering
  \includegraphics[width=\textwidth]{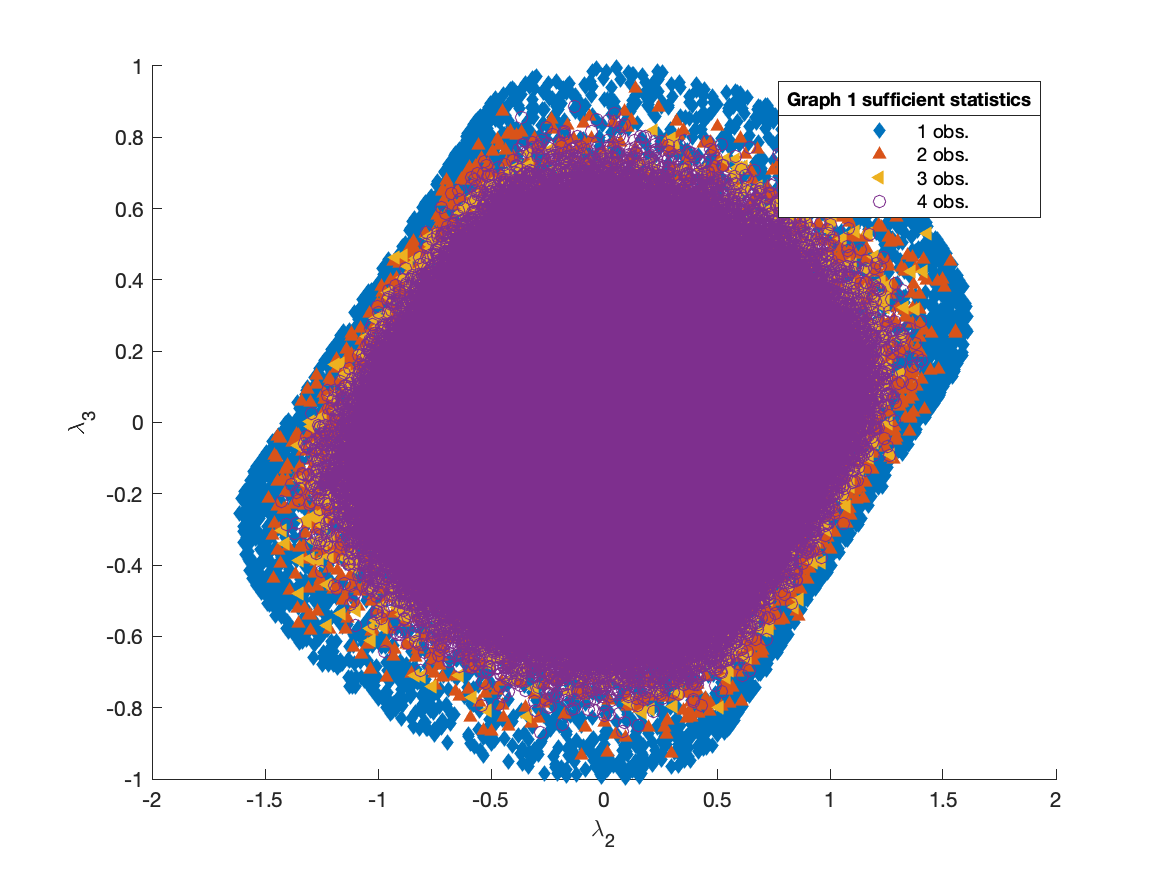}
  \caption{Cross-section of the cone of sufficient statistics for Graph~1 (all vertices in one color class), based on \cite[Table 2]{Uhler}.}
  \label{fig:graph1_suffStats}
\end{figure}

\begin{figure}[h]
  \centering
  \includegraphics[width=\textwidth]{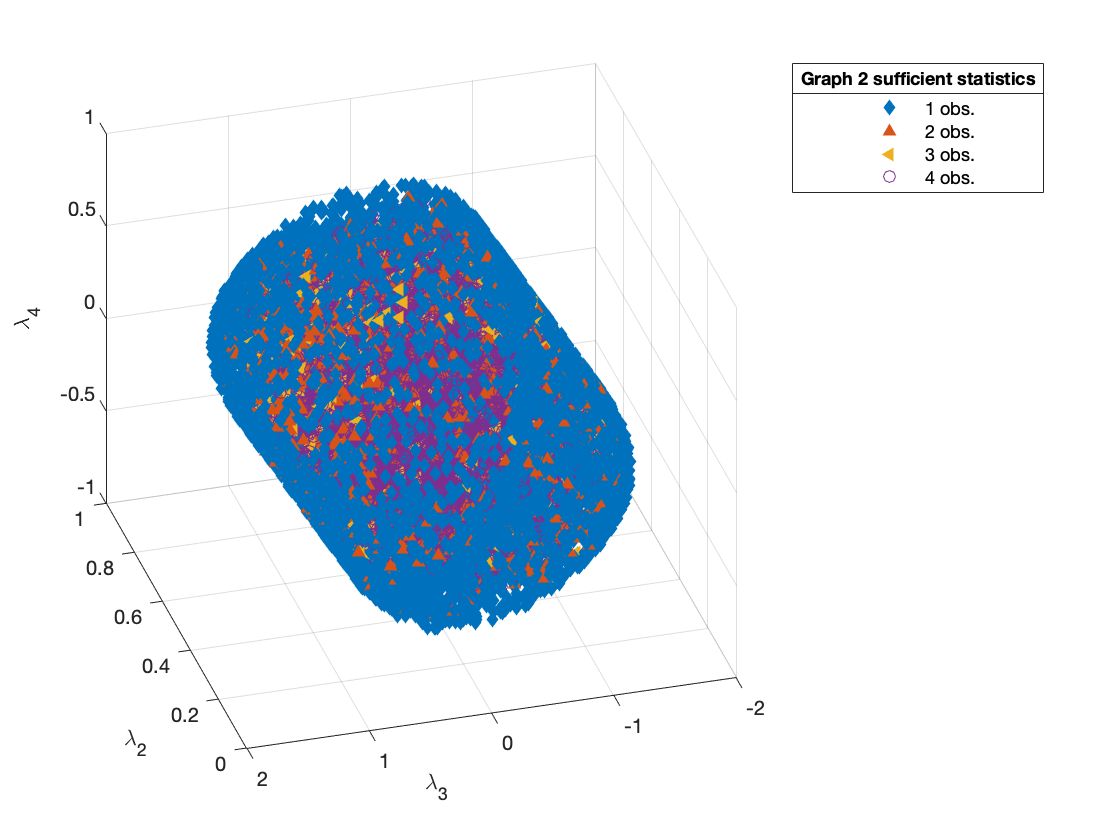}
  \caption{Cross-section of the cone of sufficient statistics for Graph~2, based on \cite[Table 2]{Uhler}.}
  \label{fig:graph2_suffStats}
\end{figure}

\begin{figure}[h]
  \centering
  \includegraphics[width=\textwidth]{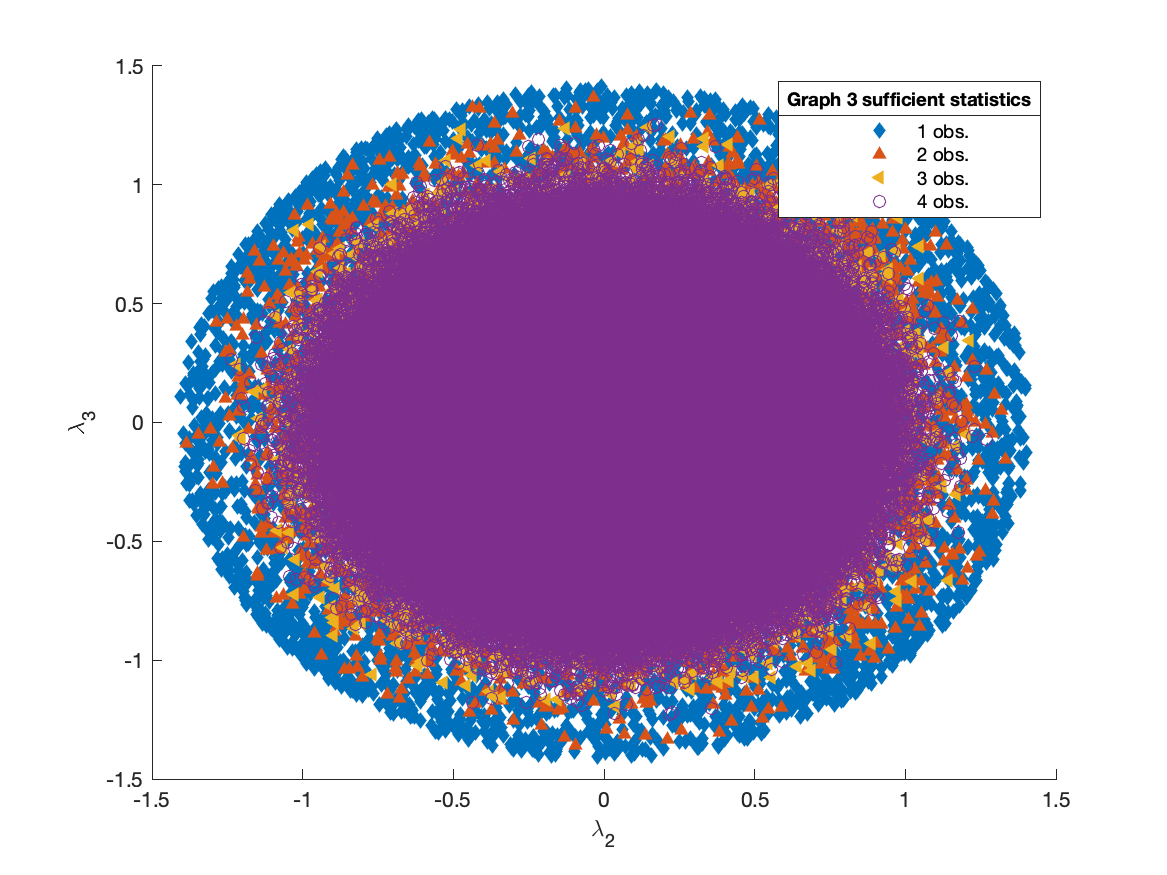}
  \caption{Cross-section of the cone of sufficient statistics for Graph~3 (all vertices in one color class), based on \cite[Table 2]{Uhler}.}
  \label{fig:graph3_suffStats}
\end{figure}

\begin{figure}[h]
  \centering
  \includegraphics[width=\textwidth]{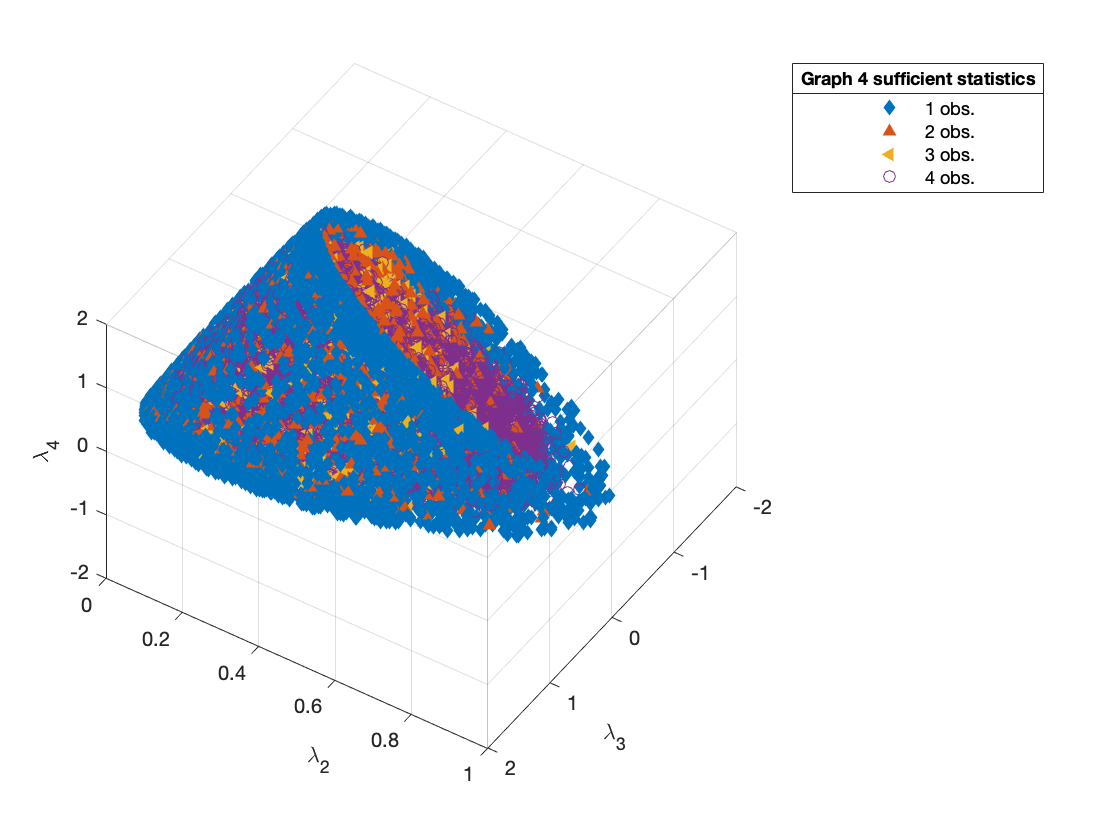}
  \caption{Cross-section of the cone of sufficient statistics for Graph~4, based on \cite[Table 2]{Uhler}.}
  \label{fig:graph4_suffStats}
\end{figure}

\begin{figure}[h]
  \centering
  \includegraphics[width=\textwidth]{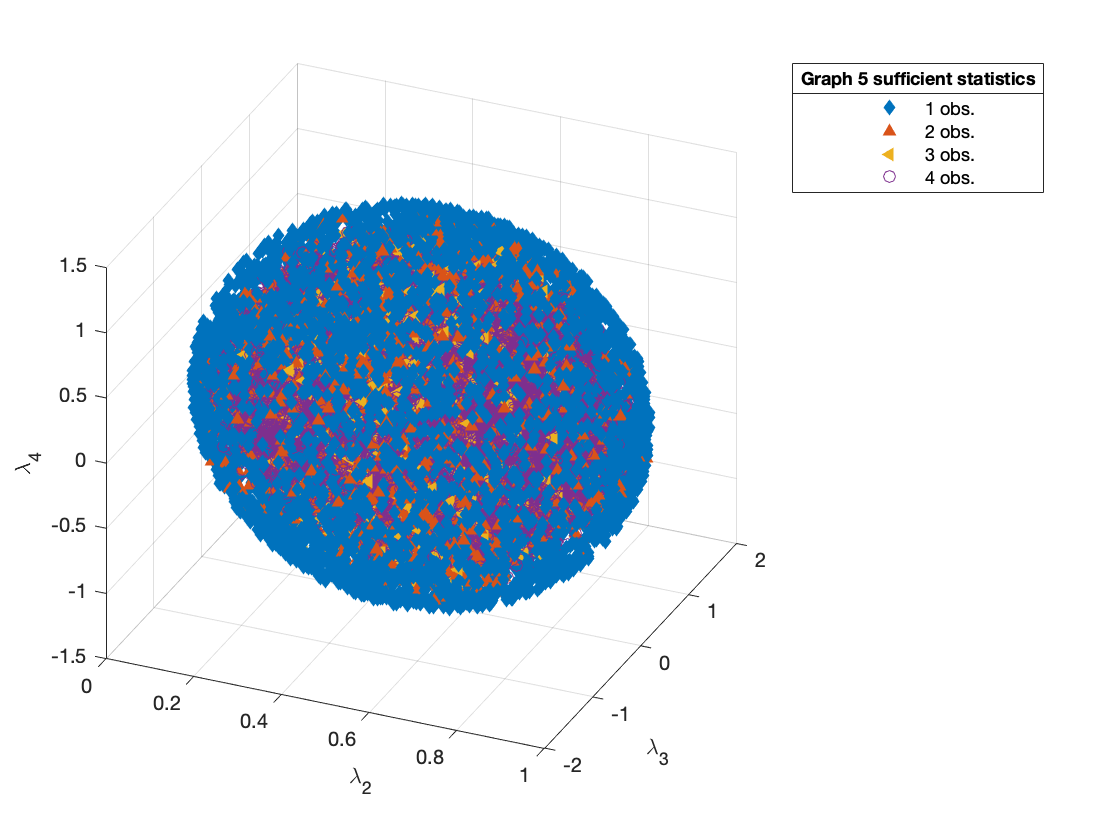}
  \caption{Cross-section of the cone of sufficient statistics for Graph~5, based on \cite[Table 2]{Uhler}.}
  \label{fig:graph5_suffStats}
\end{figure}

\begin{figure}[h]
  \centering
  \includegraphics[width=\textwidth]{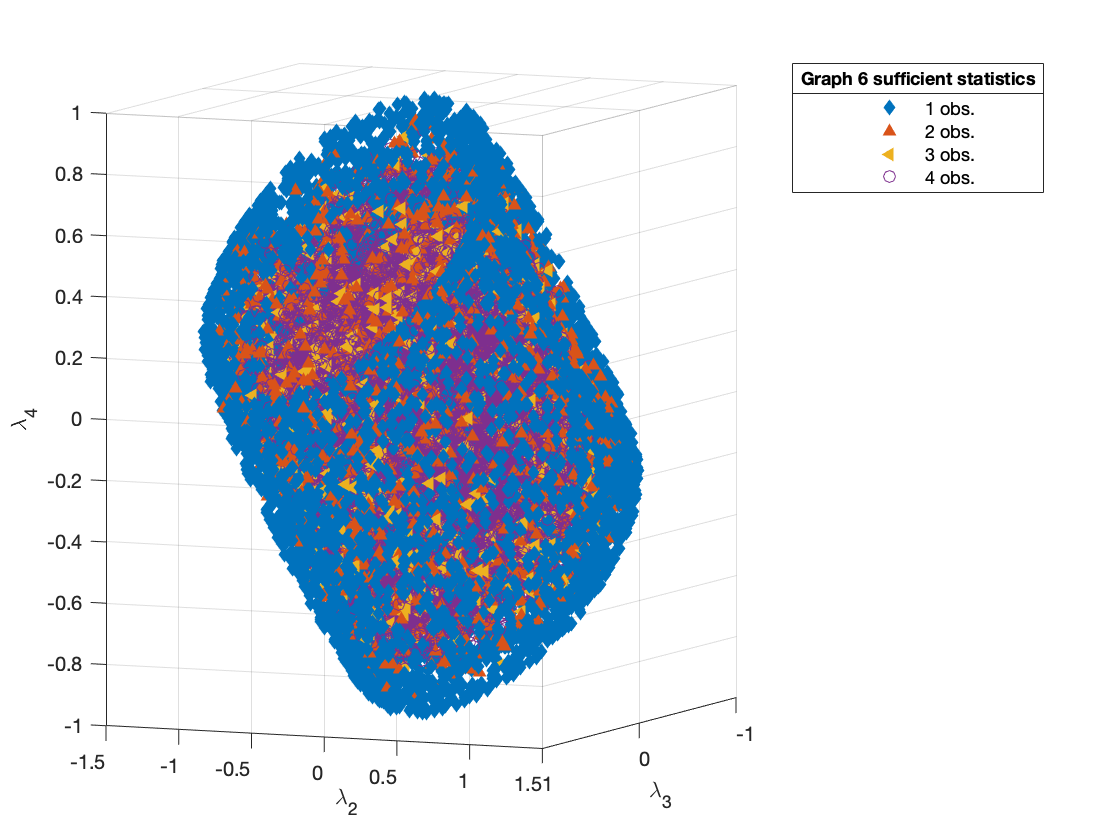}
  \caption{Cross-section of the cone of sufficient statistics for Graph~6 (all vertices in one color class), based on \cite[Table 2]{Uhler}.}
  \label{fig:graph6_suffStats}
\end{figure}

\subsection{Histogram of Cholesky parameters}\label{sec:hist}

\begin{figure}[h]
  \centering
  \includegraphics[width=\textwidth]{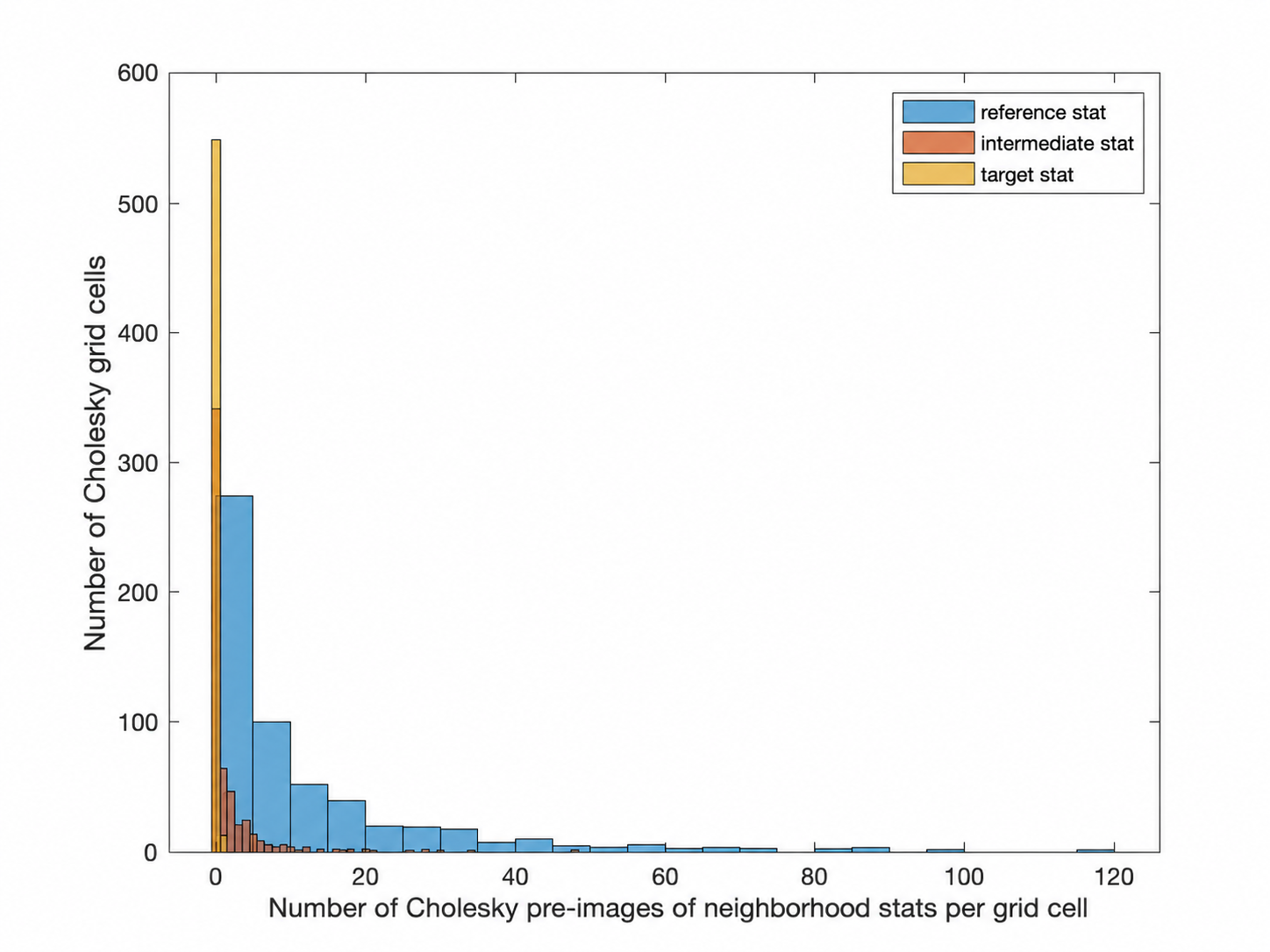}
  \caption{Distribution of deviations of the Cholesky pre-images of neighborhood points versus target, intermediate and reference stats for Graph~3 from Figure~\ref{fig:G3-SuffStats}.}
  \label{fig:graph3_hist_Cholesky}
\end{figure}

%\bibliography{socg-lipics-v2021-paper}

\end{document}